\documentclass[twoside]{article}

\usepackage{aistats_arxiv}
\usepackage[round]{natbib}

\usepackage{amsmath,amsfonts,bm,bbm}

\def\eqref#1{equation~(\ref{#1})}

\def\1{\bm{1}}

\DeclareMathAlphabet{\mathsfit}{\encodingdefault}{\sfdefault}{m}{sl}
\SetMathAlphabet{\mathsfit}{bold}{\encodingdefault}{\sfdefault}{bx}{n}

\DeclareMathOperator*{\argmax}{arg\,max}

\DeclareMathOperator{\Ind}{\mathbbm{1}}

\usepackage{tikz}

\usetikzlibrary{shapes,decorations,arrows,calc,arrows.meta,fit,positioning}
\tikzset{
    -Latex,auto,node distance =1 cm and 1 cm,semithick,
    state/.style ={ellipse, draw, minimum width = 0.7 cm},
    point/.style = {circle, draw, inner sep=0.04cm,fill,node contents={}},
    bidirected/.style={Latex-Latex,dashed},
    el/.style = {inner sep=2pt, align=left, sloped}
}

\usepackage{color}

\usepackage{amsthm}

\makeatletter
\newtheorem*{rep@theorem}{\rep@title}
\newcommand{\newreptheorem}[2]{%
\newenvironment{rep#1}[1]{%
 \def\rep@title{#2 \ref{##1}}%
 \begin{rep@theorem}}%
 {\end{rep@theorem}}}
\makeatother

\newtheorem{theorem}{Theorem}
\newreptheorem{theorem}{Theorem}

\theoremstyle{definition}
\newtheorem{definition}{Definition}
\newtheorem{assumption}{Assumption}
\newtheorem{lemma}{Lemma}
\newreptheorem{lemma}{Lemma}
\newtheorem{proposition}{Proposition}
\newreptheorem{proposition}{Proposition}

\newcommand{\indep}{\rotatebox[origin=c]{90}{$\models$}}

\usepackage{booktabs}       
\usepackage{array}
\usepackage[hyphens]{url}
\usepackage{enumitem}

\begin{document}

%
\runningtitle{Operationalizing Counterfactual Metrics: Incentives, Ranking, and Information Asymmetry}

%

\twocolumn[

\aistatstitle{Operationalizing Counterfactual Metrics: \\Incentives, Ranking, and Information Asymmetry}

\aistatsauthor{ Serena Wang \And Stephen Bates \And  P. M. Aronow \And Michael I. Jordan}

\aistatsaddress{ UC Berkeley \And  MIT \And Yale University \And UC Berkeley } ]

\begin{abstract}



From the social sciences to machine learning, it is well documented that metrics
do not always align with social welfare. 
In healthcare, \citet{dranove2003more} showed that publishing surgery mortality metrics actually harmed sicker patients by increasing provider selection behavior.
Using a principal-agent model, we analyze the incentive misalignments that arise from such \textit{average treated outcome} metrics, and show that the incentives driving treatment decisions would align with maximizing total patient welfare if the metrics \textit{(i)} accounted for counterfactual untreated outcomes and \textit{(ii)} considered total welfare instead of averaging over treated patients.
Operationalizing this, we show how counterfactual metrics can be modified to behave reasonably in patient-facing ranking systems. Extending to realistic settings when providers observe more about patients than the regulatory agencies do, we bound the decay in performance by the degree of information asymmetry between principal and agent. In doing so, our model connects principal-agent information asymmetry with unobserved heterogeneity in causal inference. 
\end{abstract}


\section{Introduction}


As machine learning (ML) is increasingly deployed in dynamic social systems with asymmetries in power and information between stakeholders, a core challenge is that the \textit{metrics} that are optimized do not always align with social welfare.
From education to healthcare to recommender systems, it has been repeatedly shown that the impact of the ability of modern ML to optimize arbitrarily complex objectives is often limited by the difficulty in choosing \textit{what} to optimize \citep{liu2023reimagining, mcnee2006being, obermeyer2019dissecting}. 


One particularly consequential example of this gap is the incentive misalignment that occurs when \textit{average treated outcome} measures are used as accountability and ranking metrics. This has led to measurable societal harm when those being ranked may selectively choose whom to treat. 
As an illustrative example, \citet{dranove2003more} showed that the publication of hospital mortality rate metrics by the New York Health Department led to dramatically worse outcomes for severely ill patients.
Specifically, hospitals had an incentive to selectively treat the healthiest patients, rather than those who would benefit most from treatment.
Today, the Centers for Medicare and Medicaid Services (CMS) continues to invest billions of dollars in the development of quality metrics \citep{wadhera2020quality,casalino2016us}. In addition to determining direct provider compensation \citep{mips}, these metrics also feed into large scale \textit{ranking systems} such as the US News and World Report and the LeapFrog Hospital Safety Score
\citep{nyt,calhospitalcompare}. At the same time, studies have continued to question the relationship between these metrics and patient outcomes~\citep[see, e.g.,][]{glance2021association,gonzalez2014hospital,ryan2009relationship,hwang2014hospital, jha2008does,smith2017dissecting}.



In this work, we directly study the incentive misalignment when \textit{average treated outcomes} are used as quality metrics. To mitigate this misalignment, we propose alternative metrics that have a foundation in causal inference. Specifically, given counterfactual estimates of patient outcomes, we outline effective usage of these estimates and discuss prevailing limitations when providers may engage in strategic behavior.

More generally, our analysis applies to any environment where metrics are learned from data over which an agent controls treatment selection. We refer to healthcare as a running example where these effects are well documented. However, other domains subject to this phenomenon include education, where student outcome metrics affect school rankings, funding, and accreditation \citep{koretz2017}; and online ranking platforms for commercial businesses like restaurants that may exercise some screening over their customers.

To study the welfare effects that arise from this dynamic interaction between quality metrics and hospitals, we employ a \textit{principal-agent model} where the \textit{principal} chooses a quality metric as a reward function, and the \textit{agent} responds by optimizing this reward function to the best of their private abilities and information. To analyze the welfare effects of metric choices, we apply a causal framework similar to that of policy learning, where the goal is to allocate treatments that maximize the total positive effects over a population relative to treating no one \citep{manski2008identification}.  
Our key ingredient is that the principal only has indirect control of the implemented policy---they may design a metric that shapes an agent's reward and hence behavior but cannot directly choose an agent's policy.  

\textbf{Contributions.} We show that \textit{average treated outcome} metrics incur unbounded regret by this definition of welfare, and show that regret can be reduced by \textit{(i)} accounting for counterfactual untreated outcomes and \textit{(ii)} considering total welfare instead of average welfare among treated patients. Applying these two simple insights yields an optimal functional form for a quality measure that achieves zero regret as long as the principal can learn the mean conditional untreated potential outcomes. We refer to this as rewarding the \textit{total treatment effect}. 
Connecting the proposed counterfactual metric to practice, we discuss two issues that arise when operationalizing the metric in real applications.
First, we study the complications that arise when the total treatment effect is used to \textit{rank} different agents that might serve different treatment populations.
Second, we consider practical issues of \textit{information asymmetry}, where the agent might observe more features about each patient than the principal.
\textit{Even an unbiased estimate of the counterfactual untreated outcome is not sufficient to maximize patient welfare when information asymmetries remain.} In addition to giving theoretical regret bounds, we also empirically show that it is not always better for a principal to condition on all known features, as this can \textit{amplify regret}. Our model yields new connections between information asymmetry in the principal-agent model and unobserved heterogeneity in causal inference.


\subsection{Related work}

Our work combines technical structure from policy learning and contract theory to analyze misalignments in accountability metrics, which have long been critiqued in the social sciences. 

\textbf{Policy learning.}
The evaluation of treatment policies by their causal effects is well established in policy evaluation and policy learning \citep{manski2008identification,hirano2009asymptotics,stoye2009minimax,athey2021policy}. We directly apply the same measures of utility and regret in this work. Building on these measures, we consider a setting where a principal (regulatory agency) can only indirectly affect the policy through measuring an agent (hospital). 
We also note that our formulation is not the only way in which strategic behavior may arise.  For example, \citet{sahoo2022policy} model patients’ strategic responses when the hospital learns a policy with limited treatment capacity. Combining this model with ours would produce an interesting pipeline analysis.

\textbf{Contract theory.}
The principal-agent model is well established in economics as a way to model incentives and equilibrium dynamics when a \textit{principal} sets a contract with rewards as a function of actions, and an \textit{agent} decides which action to take based on private information and costs \citep{laffont2009theory,gibbons2013handbook,milgrom1992economics}. Contract theory provides a structure to analyze moral hazard, where the structure of the contract and information asymmetry may lead to misalignment between the principal and agents' incentives \citep{arrow1963uncertainty,holmstrom1991multitask}. 
More recently, there is a growing recognition of the importance of the algorithmic and statistical aspects of contract theory~\citep[e.g.,][]{carroll2015robustness,tetenov2016economic,spiess2018optimal, dutting2019simple, dutting2020complexity, bates2022principalagent, alon2022bayesian}. Our work contributes to this line of thought, linking moral hazard from contract theory with causal inference, and showing how incentive and statistical considerations jointly guide the choice of accountability metrics. 

Specifically, we use this framework to analyze misalignments under a particular form of action and information asymmetry where the agent controls treatment selection, and might know more about each treatment unit than the principal. 
\citet{lazzarini2022counterfactual} applied contract theory to analyze \textit{why} regulatory agencies might lean towards outcome-based contracts (such as the average treated outcome) rather than counterfactual assessment. 
Our model differs in both goal and setup: while they model the agent's effort levels, our model considers welfare effects when the agent has general control over all treatment assignments. 

\textbf{Strategic classification.} The framework of \textit{performative prediction} describes the distribution shifts and distortion in validity that result from agent strategy in response to ML-driven decision making \citep{perdomo2020performative,hardt2016strategic}.
Related lines of work in strategic classification have considered whether classifiers incentivize agent improvement \citep{ahmadi_et_al:LIPIcs.FORC.2022.3,kleinberg2020classifiers,bechavod2020causal,haghtalab2020maximizing}, drawing connections to causality \citep{miller2020strategic,shavit2020learning}. We consider a specific strategic structure to improve metrics when agents manipulate treatment policies.

\textbf{Accountability, auditing, and measurement.}
Historically, the mismatch between accountability metrics and welfare has been well documented across domains, from monetary policy to education \citep{goodhart1984problems,campbell1979assessing,muller2019tyranny,koretz2017,mau2019metric}. Extensive work in the social sciences has critically examined accountability and auditing practices \citep{power1994audit, strathern1997improving, hoskin1996awful, rothstein2008holding}. 
Our work builds on these qualitative insights by modeling a particular misalignment that occurs when the measured party has selection power over treatment allocations, also known as ``creaming'' \citep{lacireno2002creaming}.
Measurement theory has also given both qualitative and statistical tools for understanding the validity of measurements \citep{bandalos2018measurement}, with recent extensions to fair ML \citep{jacobs2021measurement}. 
\citet{guerdan2023counterfactual} apply counterfactual modeling to estimate the measurement error when proxy outcomes are used to guide treatment decisions. Complementing this focus on measurement validity, we model the treatment incentives induced when measurements are used to reward agents, and study the resulting welfare effects.

\section{Principal-Agent Model}\label{sec:model}

To analyze the incentives and welfare effects that arise from quality metrics, we define a principal-agent model where we will refer to the organization that collects the metrics and pays or ranks the providers as the \textit{principal}, and the healthcare providers as \textit{agents}. 
First, the principal specifies a function for rewarding the agents based on their actions and the observed outcomes. In response, the agent allocates treatment decisions with knowledge of this reward function.
Our focus is on designing reward functions with high utility for the principal, which corresponds to social welfare. 

\textbf{Formal model.}
We suppose that a single agent has access to independently and identically distributed samples of characteristics $X_i \in \mathcal{X}$ for $i\in\{1,...,n\}$ treatment units (referred to generally as the \textit{treatment population}). The agent assigns a binary treatment according to treatment rule $\pi: \mathcal{X} \to [0,1]$, and we use $T^{\pi}_i \in \{0,1\}$ to denote a Bernoulli random variable indicating the realizations of the treatment rule: $\pi(X_i) = P(T^{\pi}_i = 1 | X_i)$. 
This framework allows for $T^{\pi}$ to be either stochastic or deterministic.

To model patient outcomes from treatment, we apply the potential outcomes framework \citep{neyman1923applications,rubin1974estimating}. Let $Y_i(t)$ be the potential outcome if the patient had received treatment $t \in \{0,1\}$. Let $Y_i \in \mathbb{R}$ denote the observed outcome under treatment assignments $Y_i = Y_i(T_i^{\pi})$ under the Stable Unit Treatment Value Assumption (SUTVA) \citep{rubin1980randomization}, which implies the consistency and non-interference assumptions.
Let $\tau(X) = E[Y_i(1) - Y_i(0) | X]$ denote the conditional average treatment effect given covariates $X$, and let $\mu_t(X) = E[Y_i(t) | X]$ denote the conditional mean of the potential outcome under treatment $t$. 


We suppose that the principal observes $\{X_i, T_i^{\pi}\}_{i=1}^n$ (also denoted $\bf{X}, \bf{T}^{\pi}$) for all units, and $Y_i$ for all units for which $T_i^{\pi} = 1$ (denoted $\bf{Y}$). The principal must then choose a reward function $w: \mathcal{X}^n \times \{0,1\}^n \times \mathbb{R}^n \to \mathbb{R}$ with which to reward the agent.

Turning to our behavioral assumption, we consider an agent that is risk neutral and maximizes their expected reward. That is, the agent chooses the treatment rule $\pi$ from their set of possible treatment rules $\Pi$ that maximizes $E[w(\bf{X}, \bf{T}^{\pi}, \bf{Y})]$. Let $\pi^{w} \in \arg\max_{\pi \in \Pi} E[w(\bf{X}, \bf{T}^{\pi}, \bf{Y})]$ denote this best response. In maximizing this expected reward, we assume that the agent knows $\mu_t(x)$ for all $x \in \mathcal{X}$ and $t \in \{0,1\}$. Note that there is no explicit model of the agent's cost here to keep the focus on incentives induced by accountability metrics alone, and all budget constraints are contained in the agent's feasible set of treatment rules $\Pi$. We further discuss these assumptions and possible extensions in the Appendix.

\textbf{Total welfare and regret.}
Following existing work in policy evaluation \citep{manski2008identification}, for a given treatment rule $\pi$, we define the total effect of treatment on welfare as $V(\pi) = E[Y_i(T_i^{\pi}) - Y_i(0)]$.
$V(\pi)$ is the utility of treatment rule $\pi$ relative to the alternative outcomes under no treatment \citep{manski2008identification,athey2021policy}. Thus, maximizing $V(\pi)$ also maximizes total welfare $E[Y]$ compared to treating no one. As also done in policy learning \citep{athey2021policy}, to evaluate different quality measures chosen by the principal, we define the regret for a given policy $\pi$ compared to the best feasible policy in $\Pi$ to be $R(\pi) = \max_{\tilde{\pi} \in \Pi} V(\tilde{\pi}) - V(\pi)$.
We compare different choices of quality measures $w$ by analyzing the effect on total welfare for the induced treatment rule $\pi$. In other words,
the principal's goal is to choose a reward function $w$ that leads the agent to best respond with a treatment rule with minimal regret, $R(\pi^w)$. 


\section{Comparisons of Quality Metrics}\label{sec:comparisons}

Using a principal-agent model, we formally compare different choices of quality metrics $w$ by analyzing the regret for the induced treatment rule $\pi$ given by the agent's best response. Starting with the \textit{status quo} of rewarding the average treated outcome, we show that this incurs unbounded regret. We reduce this to two main problems with the metric: \textit{(i)} lack of accounting for untreated outcomes, and \textit{(ii)} rewarding an average effect instead of a total effect. Addressing each of these in turn, we show that the regret for rewarding the \textit{average treatment effect on the treated (ATT)} is bounded but can still be high, and that rewarding the \textit{total treatment effect} finally achieves zero regret. 


{\bf Status quo: average treated outcome.} 
We begin by analyzing regret under the current common quality measure that rewards the average treated outcome, as done by the mortality measures in the New York and Philidelphia health departments in the 1990s analyzed by \citep{dranove2003more}, and in many CMS quality measures \citep{mips}. \citet{lazzarini2022counterfactual} also refers to these as ``outcome contracts.'' The reward function for the \textit{average treated outcome (ATO)} takes the following form: 

\textbf{Reward Function 1 (ATO):} \\
$w_{\text{ATO}}(\mathbf{X}, \mathbf{T}^{\pi}, \mathbf{Y}) =\begin{cases}
    \frac{\sum_{i=1}^n Y_i T_i^{\pi}}{\sum_{i=1}^n T_i^{\pi}} & {\scriptstyle \sum_{i=1}^n T_i^{\pi} > 0},\\
    0 & \text{ otherwise. } 
     \end{cases}$\\
The agent's unconstrained best response is \\$\pi^{w_{\text{ATO}}}(x) = \Ind(x \in \arg\max_{x} \mu_1(x) \text{ and } \mu_1(x) > 0)$.
\begin{proposition}[ATO Regret]\label{prop:ato_regret}
    If the conditional mean untreated potential outcomes $\mu_0(x)$ are unbounded, then the regret for the reward function $w_{\text{ATO}}$ may be arbitrarily large.
\end{proposition}
Intuitively, there are at least two failure modes that can lead to this unbounded regret. First, this reward leads the agent to ignore higher treatment effects of the patients with a lower treated outcome, such as sicker patients with higher mortality probability but more benefit from surgery. This matches the findings from \citet{dranove2003more}. Second, $w_{\text{ATO}}$ rewards agents that treat the patients with a higher treated outcome, even though treatment actually harms those patients, such as healthier patients who might incur more risks or side effects from treatment.

More broadly, there are two main problems with the construction of the ATO metric that lead to this unbounded regret. First, the lack of accounting for counterfactual outcomes leads to the two failure modes above. Second, the measure of an average outcome instead of a total outcome means that the agent will only treat the \textit{single} patient with the covariate value $x$ that \textit{maximizes} $\mu_1$. We next analyze several reasonable modifications to reward functions that address each of these problems.

{\bf Accounting for counterfactuals.}
When the principal operates with full information of the agent's selection covariates $X_i$, then regret can be reduced if they also have access to an unbiased estimator of the mean conditional untreated potential outcome.
\begin{assumption}
The principal accesses an estimator $\hat{\mu}_0(x)$ which is unbiased: $E[\hat{\mu}_0(x)] = \mu_0(x) \;\; \forall x \in \mathcal{X}$.
\end{assumption}

In general, obtaining an unbiased estimator $\hat{\mu}_0(x)$ can be difficult, but circumstances under which causal inference can be reliably conducted are well understood \citep{hernancausal}.  
As a concrete example, suppose the principal has access to an auxiliary data source $\{X'_j, T'_j, Y'_j\}_{j=1}^m$, with outcomes for untreated patients with $T'_j = 0$, collected from clinical trials or observational data. In addition to standard assumptions for identification of $\mu'_0(x) = E[Y_j'(0) | X_j' = x]$, the principal may use this data if the distribution of the conditional untreated potential outcome is also the same, $\mu'_0(x) = \mu_0(x)$, and the support of $X_j'$ covers the support of $X_i$. The difficulty of obtaining such a dataset is lessened by the fact that the principal does not require identification of the \textit{treated} potential outcome $\mu_1'(x)$ or the conditional average treatment effect $\tau'(x)$, so the treatment need not be the same. Access to additional scientific knowledge (in the form of, e.g., a more intricate structural model or functional form assumptions) can also aid in the estimation of $\mu_0(x)$. Medical research continues to develop patient risk scores using combinations of such methods \citep{sullivan2004presentation,jones1999risk} and evaluate their validity \citep{kaafarani2011validity,janssens2019validity}.

Under the perhaps optimistic assumption of access to an unbiased estimator $\hat{\mu}_0(x)$, this work focuses on effective ways for the principal to apply this estimator. It turns out that even given this unbiased estimate, principal-agent incentive misalignment can still occur, and there are still pitfalls with its downstream usage. We begin by showing how to effectively incorporate this estimator into the reward function. In later sections, we discuss effective usage in a ranking context and regret bounds under information asymmetry.

Given $\hat{\mu}_0(x)$, the principal can modify $w_{\text{ATO}}$ by directly subtracting an estimate of the untreated potential outcomes, and thus reward the \textit{average treatment effect on the treated (ATT)}:

\textbf{Reward Function 2 (ATT):}\\ 
$w_{\text{ATT}}(\mathbf{X}, \mathbf{T}^{\pi}, \mathbf{Y}) =\begin{cases}
    \frac{\sum_{i=1}^n (Y_i - \hat{\mu}_0(X_i))T_i^{\pi}}{\sum_{i=1}^n T_i^{\pi}} & {\scriptstyle \sum_{i=1}^n T_i^{\pi} > 0},\\
    0 & \text{ otherwise. } 
     \end{cases}$\\
The agent's unconstrained best response is \\$\pi^{w_{\text{ATT}}}(x) = \Ind(x \in \arg\max_{x} \tau(x) \text{ and } \tau(x) > 0)$. \\ 
The resulting regret is bounded, but still not zero.
\begin{proposition}[ATT Regret]\label{prop:att_regret}
    If $\hat{\mu}_0(x)$ is unbiased and $\pi$ is unconstrained, then the regret for the reward function $w_{\text{ATT}}$ is upper bounded as \\$R(\pi^{w_{\text{ATT}}}) \leq \max_{\pi \in \Pi} V(\pi)$.
\end{proposition}
Now that the reward function accounts for the counterfactual untreated outcome, Proposition \ref{prop:att_regret} shows that the regret cannot exceed the maximum utility. This is notably not true of the $w_{\text{ATO}}$, where the regret can be arbitrarily high due to the agent sometimes treating those with a negative treatment effect. Still, while accounting for untreated potential outcomes avoids treating those with a negative treatment effect, the ATT as a reward function still suffers from misalignment with total welfare due to the fact that it rewards the \textit{average} effect rather than the \textit{total} effect. This means that in the best response, the agent only treats patients with the single value $x$ with maximum treatment effect $\tau(x)$.

{\bf Rewarding total effects.}
To expand the agent's treatments to cover \textit{all} individuals who would benefit, we modify the above reward function by simply removing the denominator, thus rewarding a \textit{total} effect instead of an \textit{average} effect. 
This yields a reward function for the \textit{total treatment effect (TT)}.

\textbf{Reward Function 3 (TT):} \\$\quad w_{\text{TT}}(\mathbf{X}, \mathbf{T}^{\pi}, \mathbf{Y}) = \sum_{i=1}^n (Y_i - \hat{\mu}_0(X_i)) T_i^{\pi}$.

The agent's unconstrained best response is \\$\pi^{w_{\text{TT}}}(x) = \Ind(\tau(x) > 0)$, which yields zero regret.
\begin{proposition}[TT Regret]\label{prop:tt_regret}
    If $\hat{\mu}_0(x)$ is unbiased, then the regret is $R(\pi^{w_{\text{TT}}}) = 0$.
\end{proposition}
The regret is zero regardless of the feasible set $\Pi$. Thus, with two modifications to the status quo $w_{\text{ATO}}$, a quality measure $w_{\text{TT}}$ can be constructed that is aligned with total welfare, as long as the principal has access to an unbiased estimator $\hat{\mu}_0(x)$. 

\section{Ranking with Multiple Agents}\label{sec:ranking}
In Section \ref{sec:comparisons} we've shown that rewarding the total treatment effect leads the agent to maximize total welfare. While this theory applies cleanly in isolation, in real systems, quality measures are often further employed to \textit{rank} hospitals \citep{nyt,calhospitalcompare,smith2017dissecting}. It turns out that even $w_{\text{TT}}$ can exhibit problems as a ranking measure when different hospitals have different treatment population sizes and distributions.
To apply reward functions as ranking measures, we show that the total treatment effect can be modified to a more general form that allows for reweighting by covariates $X$ while still preserving incentive alignment. By reweighting the reward function relative to a reference covariate distribution, we show that the resulting quality measure leads to better hospitals receiving better rankings. 

Notationally, discussion of ranking requires extending our setting to account for multiple agents. Suppose each agent $k \in \{1,...,K\}$ observes its own sample of $n_k$ patients with covariates drawn i.i.d.\ from distribution $P_{X^{k}}$ with the same support $\mathcal{X}$. Each agent has different treatment effects, denoted by $\mu^{k}_t(x)$ and $\tau^{k}(x)$. For rankings to be meaningful, we assume that the untreated potential outcome is the same for all $k$: $\mu^{k}_0(x) = \mu_0(x)$ for all $k$. In short, one provider not treating a patient is equivalent to another provider not treating the same patient. 
Extending the principal's action space to the multi-agent ranking setting, the principal publishes score functions $\{w_k\}_{k=1}^K$, and each agent $k$ best responds individually with their own treatment policy $\pi_k$. The agents are then ranked from highest to lowest score function values. We expand this notation in the Appendix. 

\paragraph{Defining desirable ranking properties.}

For any regulatory agency or potential patient that would utilize these rankings, a clear desirable property would be that \textit{better hospitals should be ranked higher}. From the agents' perspectives, this property may also make the scores feel more ``fair.'' We formally define this property with two different degrees of strictness for the meaning of ``better.''

First, we define ``better'' as an agent having \textit{uniformly} higher treatment effects for all possible covariate values $x$, such that any patient would be better off being treated by this agent.
\begin{definition}[Uniform Rank Preservation]\label{def:uniform_ranking}
    A set of score functions $\{w_k\}_{k=1}^K$ preserves treatment effect ordering uniformly over $\mathcal{X}$ if for all $j,k \in \{1,...,K\}$, $\tau^{j}(x) \geq \tau^{k}(x)$ $\forall x \in \mathcal{X}$ $\implies$   \\
    $\underset{\pi_j \in \Pi_j}{\max} E[w_j(\mathbf{X}^{k},\mathbf{T}^{\pi_j}, \mathbf{Y}^{j})] \geq \underset{\pi_k \in \Pi_k}{\max} E[w_k(\mathbf{X}^{k},\mathbf{T}^{\pi_k}, \mathbf{Y}^{k})]$
\end{definition}
A more relaxed version of the uniform rank preservation requirement is one where an agent is ``better'' if it has higher treatment effects on \textit{average} over a reference covariate population $P_{X_0}$.
\begin{definition}[Relative Rank Preservation]\label{def:relative_ranking}
    A set of score functions $\{w_k\}_{k=1}^K$ preserves treatment effect ordering relative to a reference population $P_{X_0}$ with support $\mathcal{X}$ if for all $j,k \in \{1,...,K\}$, \\ $E[\tau^{j}(X_0)] \geq E[\tau^{k}(X_0)] \implies $ \\$\underset{\pi_j \in \Pi_j}{\max} E[w_j(\mathbf{X}^{k},\mathbf{T}^{\pi_j}, \mathbf{Y}^{j})] \geq \underset{\pi_k \in \Pi_k}{\max} E[w_k(\mathbf{X}^{k},\mathbf{T}^{\pi_k}, \mathbf{Y}^{k})]$
\end{definition}
This relative definition requires explicitly defining a reference population, which calls for careful consideration of policy goals and societal needs. Any set of scores will implicitly prioritize some populations, and calling attention to this as an explicit part of the ranking properties induced by quality measures can help policymakers more intentionally align their choices with policy goals.

\paragraph{Satisfying desirable ranking properties.}

We now formally show that $w_{\text{TT}}$ as written does not directly satisfy these ranking properties.


\begin{proposition}\label{prop:tt_ranking_violated}
    If $w_k$ is directly given by $w_{\text{TT}}$ for each agent $k$,
    then both ranking properties in Definitions \ref{def:uniform_ranking} and \ref{def:relative_ranking} will be violated.
\end{proposition}
Intuitively, this breaks down because $w_{\text{TT}}$ is subject to two auxiliary effects on top of the treatment effects. First, agents with a larger treatment population $n_k$ (e.g., larger hospitals) will have higher rankings even with the same conditional average treatment effects. Second, agents with different distributions of covariates $P_{X^{k}}$ but the same conditional average treatment effects will also end up with different rankings if some covariate values are ``easier'' to treat than others. 

To mitigate these auxiliary effects, we show that there exists a general modular form of $w_k$ that preserves the zero regret property for individual agent best responses. This general form can then be tailored to satisfy desirable ranking properties.
In particular, consider the following weighted total treatment effect reward function:

\textbf{Reward Function 4 (Weighted TT):} \\
 $\quad w_{\text{TT}}^g(\mathbf{X}, \mathbf{T}^{\pi}, \mathbf{Y}) =  \sum_{i=1}^n (Y_i - \hat{\mu}_0(X_i))T_i^{\pi} g(X_i)$.

Any reward function in this family induces the desired agent best response.
\begin{theorem}[Incentive Alignment]\label{thm:reweighting}
    Suppose $\hat{\mu}_0(x)$ is unbiased, and $\pi$ is unconstrained. For any function $g: \mathcal{X} \to \mathbb{R}^+$, $w_{TT}^g$ yields an agent best response with regret $R(\pi^{w_{\text{TT}}^g}) = 0$. 
\end{theorem}
Theorem \ref{thm:reweighting} shows that reweighting the reward function by any function of the covariates $X$ does not hurt incentive alignment. Thus, the principal may choose functions $g_k$ for each agent to achieve desirable ranking properties. Specifically, setting $g_k$ to reweight each agent's covariate distribution to the reference distribution $P_{X_0}$ satisfies both ranking properties in Definitions \ref{def:uniform_ranking} and \ref{def:relative_ranking}.


\begin{theorem}[Ranking Desiderata Satisfied]\label{thm:ranking_reweighted}
    Let $P_{X^{k}}$ be absolutely continuous with respect to $P_{X_0}$, and let $g_k = \frac{1}{n_k}\frac{dP_{X_0}}{dP_{X^{k}}}$ be the normalized Radon–Nikodym derivative of the reference distribution $P_{X_0}$ with respect to agent $k$'s covariate distribution $P_{X^{k}}$. Then setting $w_k$ to be $w_{\text{TT}}^{g_k}$ for agent $k$'s treatment population satisfies both ranking properties in Definitions \ref{def:uniform_ranking} and \ref{def:relative_ranking} as long as $\Pi_k$ is unconstrained and treatment effects are nonnegative, $\tau^{k}(x) \geq 0$, for all $k \in \{1,...,K\}$.
\end{theorem}
Theorem \ref{thm:ranking_reweighted} shows that a simple distributional reweighting can achieve the desirable ranking properties that preserve treatment effect ordering both uniformly over all $\mathcal{X}$ and relatively to some reference population $P_{X_0}$. In practice, if the exact Radon–Nikodym derivatives are not known, the importance sampling literature contains many techniques for estimating expectations with distributional reweighting \citep{mcbook}.
Overall, this simple reweighting modification of the total treatment effect score function addresses important policy considerations when quality measures are used for ranking.
\section{Information Asymmetry}\label{sec:info_asymmetry}

So far, the incentive alignment and ranking properties have relied on the assumption that both principal and agent operate with the same covariate information $X$. In practice, \citet{dranove2003more} remark that ``providers may be able to improve their ranking by selecting patients on the basis of characteristics that are unobservable to the analysts but predictive of good outcomes.'' 
Under such information asymmetry, we show that even the optimistic assumption of an unbiased estimator $\hat{\mu}_0(x)$ is not enough to guarantee zero regret. We both upper and lower bound regret in terms of the additional \textit{heterogeneity} observed by the agent. 

Suppose the agent observes additional covariates $U_i \in \mathcal{U}$, and selects a treatment rule $\pi: \mathcal{X}, \mathcal{U} \to [0,1]$. Suppose the principal still observes only $\{X_i, T_i^{\pi}, Y_i\}_{i=1}^n$, and chooses a reward function $w(\mathbf{X}, \mathbf{T}^{\pi}, \mathbf{Y})$ that does not depend on $U_i$. Let $\mu_t(X,U) = E[Y_i(t) | X,U]$ and $\tau(X,U) = E[Y_i(1) - Y_i(0) | X, U]$. The utility and regret are still defined as in Section \ref{sec:model}.


Applying the optimal reward function from the full information setting in Section \ref{sec:comparisons}, suppose the principal rewards the agent with the total treatment effect $w_{\text{TT}}$. As in Section \ref{sec:comparisons}, suppose the principal applies an unbiased estimator $\hat{\mu}_0(X)$ of the untreated potential outcome conditional on $X$, with $E[\hat{\mu}_0(x)] = \mu_0(x)$ $\forall x \in \mathcal{X}$. Note that the principal identifies $\mu_0(X)$, but \textit{not} $\mu_0(X, U)$. 
We show that the regret is bounded if the effect of the agent's private information $U_i$ on the untreated potential outcomes is bounded. 


\begin{assumption}[Bounded Heteogeneity]\label{assumption:unobserved-confounding-avg-historical} 
The effect of $U_i$ on the conditional untreated potential outcome is bounded as $E[ | \mu_0(X_i) - \mu_0(X_i,U_i)| ] \leq \gamma_{\text{marg}}$.
\end{assumption}

Note that $\mu_0(X_i) = E[\mu_0(X_i, U_i) | X_i]$. Thus, the heterogeneity bound is akin to bounding a statistical error between the conditional untreated potential outcome $\mu_0(X_i,U_i)$ known to the agent, and ``marginal'' $\mu_0(X_i)$ estimated by the principal. We can both upper and lower bound the regret in terms of this error:

\begin{theorem}[Regret With Information Asymmetry]\label{thm:historical_regret}
    Suppose $\hat{\mu}_0(x)$ is unbiased. Under Assumption \ref{assumption:unobserved-confounding-avg-historical}, the regret is upper bounded as $R(\pi^{w_{\text{TT}}}) \leq 2\gamma_{\text{marg}}$.
\end{theorem}

This upper bound is tight up to a linear constant.
\begin{proposition}\label{prop:historical_regret_tight}
    $\forall \epsilon > 0$, there exist distributions of $X_i, U_i, Y_i(0), Y_i(1)$ 
    wherein $R(\pi^{w_{\text{TT}}}) \geq \gamma_{\text{marg}} - \epsilon$. 
\end{proposition}
Thus, this notion of \textit{heterogeneity} is key to determining regret under information asymmetry. As a realistic example of cases when Assumption \ref{assumption:unobserved-confounding-avg-historical} might be satisfied, studies of cardiovascular disease risk have shown that “the magnitude of risk related to smoking is far larger than any ostensible benefit related to moderate drinking” \citep{mukamal2006effects}. Thus, if $U$ were some attribute for which the relative effect on top of $X$ was small, then $\gamma_{\text{marg}}$ would be small. On the other hand, \citet{rodgers2019cardiovascular} report that 
sex hormones and diabetes have compounding effects on cardiovascular disease risk. If $X$ and $U$ have strong compounding effects on $Y(0)$, then $\gamma_{\text{marg}}$ could be large. 

\paragraph{Information asymmetry and confounding.} Information asymmetry relates closely to the possibility of confounding bias in the estimator $\hat{\mu}_0(x)$. First, the agent's knowledge of $U$ could mean that the data source from which the principal estimated $\hat{\mu}_0(x)$ was also confounded by $U$. In this case, the literature on sensitivity analysis and policy learning with unobserved confounding proposes a range of robust estimates for $\mu_0(x)$~\citep[see, e.g.,][]{yadlowsky2022bounds,kallus2018confounding}. In our setting, robust estimation of $\mu_0(x)$ is not enough, since the agent's treatment rule can depend on $U$. Still, the minimax techniques from these works 
may be a useful avenue for designing future robust reward functions $w$.
Second, information asymmetry can exacerbate confounding if the agent were able to directly affect the principal's estimator $\hat{\mu}_0(x)$, which may happen if, e.g., the principal were to estimate $\hat{\mu}_0(x)$ from the agent's untreated units with $T^{\pi}_i = 0$. We discuss this case in detail in the Appendix, and show that a stronger assumption yields a similar bound to Theorem \ref{thm:agent_regret}.  
Most importantly, our key result is that even without confounding, information asymmetry \textit{still} causes problems via the agent's ability to discriminate on $U$.

\section{Experiments}\label{sec:experiments}

We turn to several clinical datasets to 
evaluate the welfare impacts of different quality metrics under different conditions of information asymmetry. We 
show empirically that the regret incurred by $w_{\text{ATO}}$ can be high. We also show that under information asymmetry, regret can be \textit{amplified} if the principal estimates $\hat{\mu}_0(x)$ conditioned on some subsets of features. With careful feature selection for $x$, regret may be reduced.

\textbf{Horse Colic dataset.}
The Horse Colic dataset from the UCI repository \citep{Dua:2019} contains $n = 300$ horse colic cases. Horses were either treated with surgery ($T = 1$) or not ($T = 0$), with a treatment rate of $0.6$. The $20$ covariates include each horse's age and presenting symptoms such as abdominal distension, pulse, blood test results, etc. (see Appendix for a full list). Our outcome of interest $Y$ is whether the horse lived ($Y = 1$) or died ($Y = -1$). 
To approximate the mean conditional potential outcomes, we apply a logistic model with interaction terms between the treatment and covariates: $P(Y(t) = 1 | X = x) = \sigma(\beta_0 + \beta_1 x + \beta_2 t + \beta_3 xt)$.
We estimate the parameters $\beta$ on the dataset using logistic regression, and take these as given to produce $\mu_0(x)$ and $\mu_1(x)$. The fitted $\mu_0(x)$ and $\mu_1(x)$ show $62$ horses benefiting from surgery and $146$ being better off without surgery. On the horses that would benefit, the average benefit was $0.147$, which is fairly significant. The clinical validity of these estimated potential outcomes cannot be verified from this data alone, and we instead take these estimates as synthetic potential outcomes.

\textbf{International Stroke Trial dataset.}
We also consider data from the \citet{international1997international}, which was a randomized trial studying the effects of drug treatments in acute stroke. \citet{kallus2018confounding} studied this dataset in a different policy learning setting, and we apply a similar setup by comparing treating with high doses of heparin and aspirin $(T = 1)$ with aspirin alone ($T = 0$). This leaves $n = 7264$ patients and a treatment rate of $0.33$. The $20$ covariates include each patient's age, sex, and clinical symptoms such as prior stroke types and complications. Like \citet{kallus2018confounding}, we consider a scalarized outcome score $Y \in [-4, 3]$ that accounts for patient outcomes including death, recovery, and side effects at 14 days and 6 months after treatment (details in the Appendix). 
We approximate the mean conditional potential outcomes using a linear model with interaction terms between treatment and covariates:
$E[Y(t) | X = x] = \beta_0 + \beta_1 x + \beta_2 t + \beta_3 xt$.
We fit an OLS estimate of $\beta$, and use the resulting $\mu_0(x)$ and $\mu_1(x)$ functions as synthetic mean conditional potential outcomes. The fitted $\mu_0(x)$ and $\mu_1(x)$ showed $1360$ patients benefiting from heparin and $5904$ being better off without heparin. On the patients that would benefit, the average benefit of treatment was $0.025$. This tracks with the study's findings that the benefit of heparin was non-significant and inconclusive. 

\begin{table*}[!ht]
  \caption{Utility and regret comparisons for different reward functions. For each reward function $w$, we report utility $V(\pi^w)$, regret $R(\pi^w)$, and the realized treatment rate $P(T^{\pi^w} = 1)$.}
  \label{tab:regrets_full_info}
  \centering
  \begin{tabular}{lllllll}
    \toprule
    & \multicolumn{3}{c}{Horse Colic dataset}   &  \multicolumn{3}{c}{Stroke Trial dataset}     \\
    \cmidrule(r){2-7}
    Reward function & Utility  & Regret  & Treat rate & Utility & Regret & Treat rate \\
    \midrule
    $w_{\text{ATO}}$ & $0.00000$ & $0.1477$  & $0.1922$ & $0.00004$ & $0.0251$ & $0.0001$ \\
    $w_{\text{ATT}}$ & $0.00784$ & $0.1399$ & $0.0039$ & $0.00013$ & $0.0250$ & $0.0001$ \\
    $w_{\text{TT}}$ & $0.14774$ & $0.000$ & $0.2431$ & $0.02518$ & $0.000$ &  $0.1872$\\
    \midrule
    $w_{\text{TT}}$ (no info) & $0.10008$ & $0.0476$ & $0.6235$ & $-0.04888$ & $0.0741$ &  $0.4829$\\
    $w_{\text{TT}}$ (demographic info) & $0.10008$ & $0.0476$ & $0.6235$ & $-0.06391$ & $0.0891$ & $0.5041$ \\
    \bottomrule
  \end{tabular}
\end{table*}

\begin{figure}[!ht]\label{fig:ist_gammas_regrets_cumul}
  \centering
  \setlength\extrarowheight{-20pt}
  \begin{tabular}{cc}
    $\gamma_{\text{marg}}$ per feature \\
\begin{minipage}[t]{.45\textwidth}
\vspace{0pt}
\raggedleft
\includegraphics[trim=0 168 5 5,clip,width=0.99\textwidth]{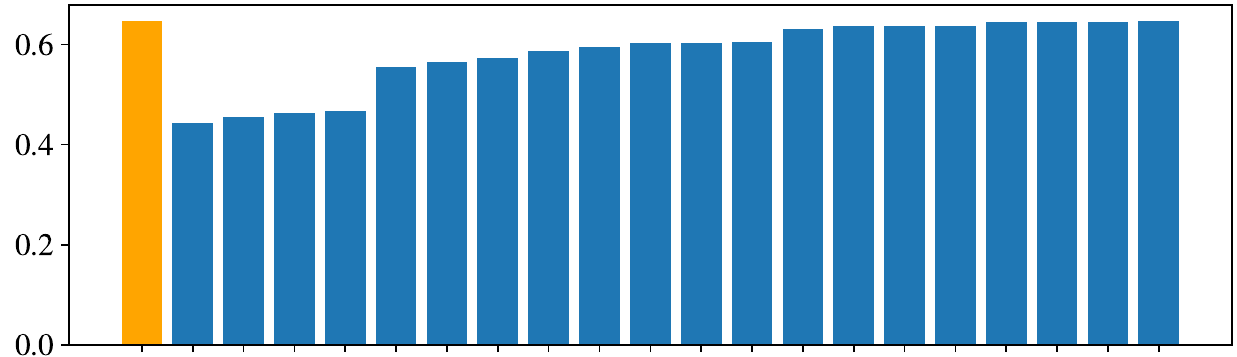}
\vspace{-10pt}
\end{minipage} \\
 Regret for increasing feature sets\\
\begin{minipage}[t]{.45\textwidth}
\vspace{0pt}
\raggedleft
\includegraphics[trim=0 8 5 5,clip,width=\textwidth]{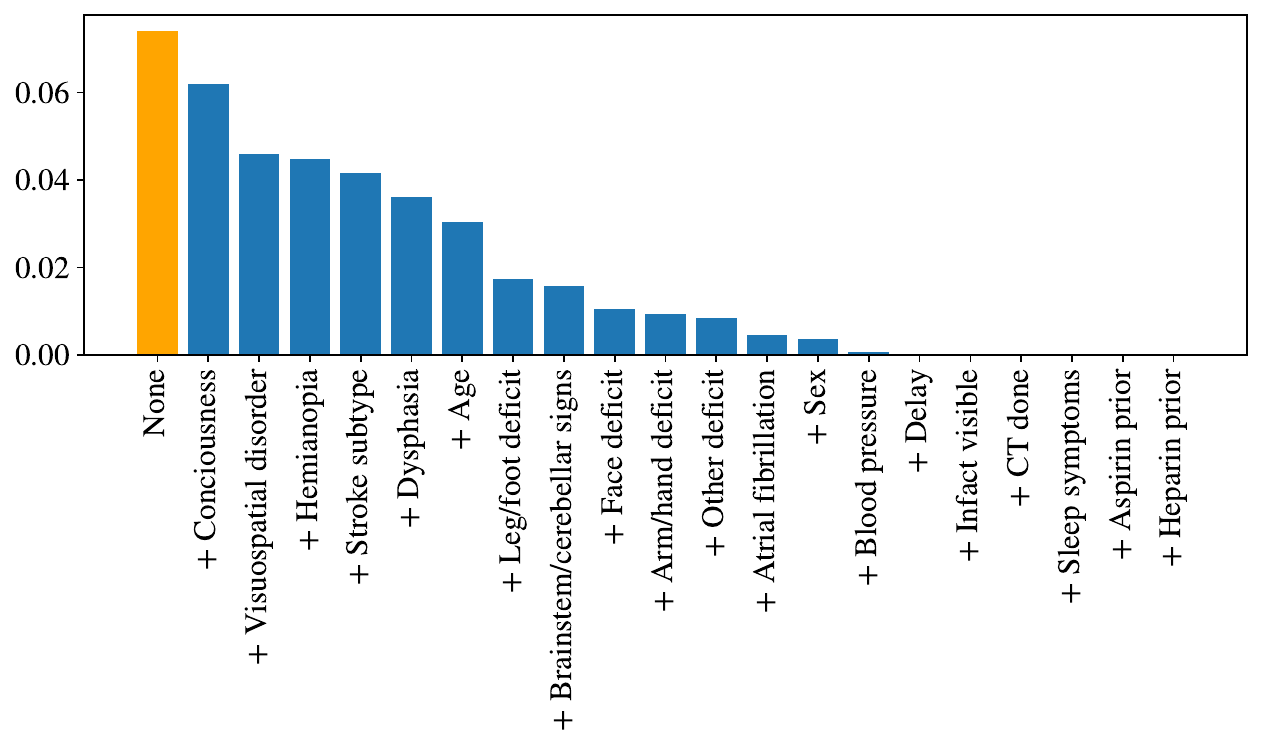}
\vspace{0pt}
\end{minipage}\vspace{-10pt}
\\
  \end{tabular}
  
  \caption{Regret under fine-grained information asymmetry on the Stroke Trial dataset (Horse Colic in the Appendix). The \textit{top} plot shows $\gamma_{\text{marg}}$ values if the principal only knows each individual feature. The \textit{bottom} plot shows regret as the principal accumulates features from the left (most important).}
\end{figure}

\textbf{Results and discussion.}
We first compare the reward functions from Section \ref{sec:comparisons} by calculating the utility and regret empirically over the dataset using the synthetic potential outcomes. Table \ref{tab:regrets_full_info} shows that for both datasets, the utility for $w_{\text{TT}}$ is positive and higher for the Horse Colic dataset than for the Stroke Trial dataset, which tracks with the clinical finding that the heparin treatment 
did not have significant effect. The utility for $w_{\text{ATO}}$ is close to zero for both datasets. 

Next, we consider the effect of information asymmetry on regret. The last two rows of Table \ref{tab:regrets_full_info} show the regret when the principal applies $w_{\text{TT}}$, but either observes no covariates (``no info'') and identifies $E[Y(0)]$, or only observes age and/or sex (``demographic info'') and identifies $E[Y(0) | \text{demographics}]$. While conditioning on age has no effect on the Horse Colic dataset, on the Stroke Trial dataset, conditioning on demographics actually \textit{hurts} utility compared to using ``no info.'' This suggests that it is not always better for the principal to condition on all known information, and thus policymakers should exercise caution in designing stratified quality metrics.   
For example, CMS currently measures age-specific kidney transplant rates for organ procurement organizations (OPOs) \citep{cms_opo}, and our findings question the value of such incomplete stratification under OPO treatment selection.
Theoretically, this finding structurally mirrors the phenomenon of \textit{bias amplification} when estimating causal effects with unobserved confounding, where conditioning on more observed features can actually increase bias \citep{pearl2012class}. Here we observe \textit{regret amplification}, and we encourage replication of similar analyses by regulatory agencies with internal data sources.


We also study a more continuous spectrum of information asymmetry by showing regret as the principal accumulates increasingly large subsets of the available features. Relating this to Section \ref{sec:info_asymmetry}, we first sort the features in ascending order of feature importance, as measured by estimating $\gamma_{\text{marg}}$ when the principal knows only the individual feature. Then, Figure \ref{fig:ist_gammas_regrets_cumul} shows regret when the principal knows increasingly large feature subsets, building up starting from the most ``important'' feature. Regret is reduced significantly after accounting for less than half of the features. 
In practice, the principal would not know true values of $\gamma_{\text{marg}}$. However, approximations of $\gamma_{\text{marg}}$ may serve as a reasonable heuristic for feature selection when a regulator has a large set of known covariates, but needs to prune them for interpretability or cost.


\section{Conclusions}
We have studied the harm to social welfare that occurs when accountability metrics are not aligned with utility. Even under optimistic assumptions about the availability of an unbiased counterfactual estimate, the potential for regret still exists under information asymmetry with treatment effect heterogeneity.
Given the compounded difficulty of estimating causal effects on top of our consideration of treatment incentives and ranking, we recommend that designers exercise caution, humility, and vigilance in their construction of metrics.
The task is difficult, but we have established the contours of one potentially fruitful approach.

\newpage
\clearpage

\Urlmuskip=0mu plus 1mu\relax
\bibliographystyle{plainnat}
\bibliography{references}

\newpage
\clearpage
\appendix
\onecolumn
\aistatstitle{Supplementary Materials for Operationalizing Counterfactual Metrics: Incentives, Ranking, and Information Asymmetry}
\section{Estimating the Untreated Potential Outcome}\label{sec:estimating_mu_0}

Our main paper considers incentive problems under the assumption that the principal has access to measures of patient risk through an unbiased estimator of the untreated potential outcome, $\hat{\mu}_0(x)$ with $E[\hat{\mu}_0(x)] = \mu_0(x)$.

Obtaining such an unbiased estimator can be difficult. Still, the conditions under which causal estimation can be done are well understood. In this section, we give some examples of sufficient conditions for identification of $\mu_0(x)$. This is not an exhaustive list, but rather a starting point for analysis of existing data sources.

\subsection{Examples of Data Sources}

We list two examples of data sources that may contribute to building the estimator $\hat{\mu}_0$, along with sufficient conditions for identification. These may not always be realistic, and are also not the only possible conditions for identification. In reality, practitioners may also be able to leverage knowledge of a more detailed structure causal model, or functional form assumptions (see \citet{hernancausal} for a more thorough coverage of methodologies and assumptions).

\begin{enumerate}
    \item \textbf{Auxiliary untreated data.} Suppose the principal has access to an auxiliary dataset $\{X_j', T_j', Y_j'\}_{j=1}^m$. Then the principal may produce an unbiased estimator $\hat{\mu}_0(x)$ for $\mu_0(x)$ from this auxiliary dataset if:
    \begin{enumerate}
        \item $\mu_0'(x) = E[Y_j'(0) | X'_j = x]$ is identifiable from this dataset. A set of sufficient conditions for this would be if ignorability was satisfied, $\{Y_j'(0), Y_j'(1)\} \indep T_j' | X_j'$, and $P(T_j' = 0 | X_j = x) > 0 \; \forall x \in \mathcal{X}$, and SUTVA was satisfied, $Y_j' = Y_j'(T_j')$ (encompassing both consistency and non-interference).
        \item The relationship between the untreated potential outcome $Y_j'(0)$ and the covariates $X_j'$ is the same as in the treatment population between $Y_i(0)$ and $X_i$. That is, $E[Y_j'(0) | X'_j = x] = E[Y_i(0) | X_i = x]$.
        \item The support of $X'_j$ covers the support of $X_i$. Formally, the distribution of $X_i$ is absolutely continuous with respect to the distribution of $X'_j$.
    \end{enumerate}
    
    \item \textbf{Untreated patients in the treatment population.} In the main paper, we have assumed that the principal only observed outcomes $Y_i$ for treated units with $T_i^{\pi} = 1$. If the principal also observed a fraction of outcomes for untreated units, with $T_i^{\pi}=0$, then these may also be incorporated into the estimate for $\hat{\mu}_0(x)$. Sufficient conditions for identifying $\mu_0(x)$ from a dataset of the agent's untreated units are:
    \begin{enumerate}
        \item No information asymmetry: the agent's treatment policy $\pi$ only depends on $X_i$. Under this assumption, we have no confounding: $\{Y_i(0), Y_i(1)\} \indep T_i^{\pi} | X_i$.
        \item Positivity: $P(T_i^{\pi} = 0 | X_i = x) > 0$ for all $x \in \mathcal{X}$, and $Y_i$ is observed for a nonzero proportion of untreated units for all $x$. 
    \end{enumerate}

    The positivity assumption is particularly tricky. In the principal-agent game, there is no guarantee that the agent's best response policy $\pi$ would satisfy positivity. Positivity may be enforced by restricting the treatment rule class $\Pi$ to only include treatment rules where $\pi(x) < 1$, but this may not always be possible or ethical from a policy standpoint. In practice, it may be possible to create a composite dataset that combines data from both sources. That way, data from untreated patients in the treatment population could supplement auxiliary untreated data to produce a better estimator $\hat{\mu}_0(x)$ than using auxiliary untreated data alone.
\end{enumerate}

\subsection{Necessity of Estimating the Untreated Potential Outcome}

Here we address the necessity of including an estimate of the untreated potential outcome in the reward function $w$. We show that if $w$ is completely unable to differentiate between different distributions over the untreated potential outcome, perhaps through incorporating some estimate of some function of $Y_i(0)$ or through other constraints (e.g. co-monotonicity of $\mu_1(x)$ and $\mu_0(x)$), then regret is necessarily unbounded.

Formally, let $\mathcal{D}$ denote a distribution over $X_i, Y_i(1), Y_i(0)$. Let $\mathcal{D}'$ denote an alternate distribution such that $E_{\mathcal{D}'}[Y_i(0) | X_i = x] = \mu_0'(x)$, and the joint distribution $X_i, Y_i(1)$ under $\mathcal{D}'$ remains unchanged. Let $Y_i$ continue to denote the observed outcome under treatment assignments $Y_i = Y_i(T_i^{\pi})$, and note that the distribution of $Y_i$ changes under $\mathcal{D}'$ as well. We show that if the reward function $w$ cannot differentiate between worlds $\mathcal{D}$ and $\mathcal{D}'$,  then regret will necessarily be unbounded.

\begin{assumption}[Non-degenerate $w$]\label{assn:non-degen-w}
    Assume that $w$ does not always induce a degenerate best response: that is, there exists a distribution $\mathcal{D}$ such that $\arg\max_{\pi \in \Pi} E_{\mathcal{D}}[w(\mathbf{X}, \mathbf{T^{\pi}}, \mathbf{Y})]$ contains $\pi^w$ with $E_{\mathcal{D}}[\pi^w(X_i)] > 0$ (in other words, there exists a non-measure-zero set $\bar{\mathcal{X}}$ with $\pi^w(\bar{x}) > 0 \; \forall \bar{x} \in \bar{\mathcal{X}}$). 
\end{assumption}

\begin{theorem}
     Suppose $w$ satisfies Assumption \ref{assn:non-degen-w}. Let $\mathcal{S}$ be the set of all pairs of distributions $\mathcal{D}, \mathcal{D}'$ such that the distributions of $X_i, Y_i(1)$ remain unchanged between $\mathcal{D}$ and $\mathcal{D}'$. If for all pairs of distributions $\mathcal{D}, \mathcal{D}' \in \mathcal{S}$, $E_{\mathcal{D}}[w(\mathbf{X}, \mathbf{T^{\pi}}, \mathbf{Y})] = E_{\mathcal{D}'}[w(\mathbf{X}, \mathbf{T^{\pi}}, \mathbf{Y})]$ for all $\pi$, then regret is unbounded.
\end{theorem}
\begin{proof}
    Let $\pi' \in \arg\max_{\pi \in \Pi} E_{\mathcal{D}'}[w(\mathbf{X}, \mathbf{T^{\pi}}, \mathbf{Y})]$, with $\pi'(\bar{x}) = \delta > 0$ for all $\bar{x} \in \bar{\mathcal{X}}$, with $\bar{\mathcal{X}}$ being a set of measure $\rho > 0$ (which exists by Assumption \ref{assn:non-degen-w}). Let $\mu_1(\bar{x}) = E[Y_i(1) | X = \bar{x}] = \alpha$ for all $\bar{x} \in \bar{\mathcal{X}}$. Let $\mathcal{D}$ be distribution such that $\mu_0(\bar{x}) = E_{\mathcal{D}}[Y_i(0) | X = \bar{x}]  = \beta$ for all $\bar{x} \in \bar{\mathcal{X}}$. Since $E_{\mathcal{D}}[w(\mathbf{X}, \mathbf{T^{\pi}}, \mathbf{Y})] = E_{\mathcal{D}'}[w(\mathbf{X}, \mathbf{T^{\pi}}, \mathbf{Y})]$ for all $\pi$, we have $\pi' \in \arg\max_{\pi \in \Pi} E_{\mathcal{D}}[w(\mathbf{X}, \mathbf{T^{\pi}}, \mathbf{Y})]$ as a candidate best response to $w$ under distribution $\mathcal{D}$.
    
    Denote the welfare maximizing policy for distribution $\mathcal{D}$ as \[\pi^{*} \in \arg\max_{\pi \in \Pi} E_{\mathcal{D}}[\pi(X_i)(\mu_1(X_i) - \mu_0(X_i))].\]
    Choose $\beta > \alpha$. Then the set of welfare maximizing policies contains a $\pi^{*}$ with $\pi^{*}(\bar{x}) = 0$ for all $\bar{x} \in \bar{\mathcal{X}}$. 
    
    The regret is then
    \[R(\pi') = E_{\mathcal{D}}[(\pi^{*}(X_i) - \pi'(X_i)) (\mu_1(X_i) - \mu_0(X_i))] \geq -\delta (\alpha - \beta) \rho. \]

    Fixing $\delta>0, \rho>0, \alpha<\infty$ and choosing $\beta$ to be arbitrarily high results in arbitrarily high regret. 
\end{proof}



\section{Proofs from Section \ref{sec:comparisons} on Reward Function Comparisons}

This section provides proofs and formal statements for the best responses and regrets corresponding to each reward function in Section \ref{sec:comparisons}. 

\subsection{Notation}
Let $\int f(Z) dP_Z$ denote integration of the function $f(Z)$ with respect to the probability measure for the random variable $Z$. Let the bold variable $\mathbf{Z}$ denote the vector of i.i.d.\ random variables $\{Z_i\}_{i=1}^n$. Let $\Ind(\cdot)$ denote an indicator function with 
\[\Ind(x \in S) = 
\begin{cases}
    1 &\text{if } x \in S\\
    0 & \text{otherwise}.
\end{cases}\]

Let $\pi^{w}$ denote the agent's best response for the reward function $w$: \[\pi^{w} = \arg\max_{\pi \in \Pi} E[w(\bf{X}, \bf{T^{\pi}}, \bf{Y})].\]

\subsection{Additional Reward Function}

A simple extension from $w_{\text{ATO}}$ would be to measure the \textit{total treated outcome (TO)}: 

\textbf{Reward Function 5 (TO):}
\begin{equation}
    \label{eq:w_total_treated_outcome}
    w_{\text{TO}}(\mathbf{X}, \mathbf{T}^{\pi}, \mathbf{Y}) = \sum_{i=1}^n Y_i T_i^{\pi}.
\end{equation}

The unconstrained best response is $\pi^{w_{\text{TO}}}(x) = \Ind(\mu_1(x) > 0)$. Like $w_{\text{ATO}}$, this incurs unbounded regret; however, this rule treats more people due to treating all individuals for whom the treated outcome is positive. Thus, there is no longer the ``creaming'' \cite{muller2019tyranny} issue of maximizing the \textit{average} effect at the expense of the \textit{total} effect. 

\subsection{Formal Statements and Proofs for Agent Best Responses}
We give formal statements and proofs for the agent best responses to the different reward functions stated in Section \ref{sec:comparisons}. 

\begin{proposition}[ATO Best Response]\label{prop:ato_best_response}
    Suppose $\Pi$ is the set of all functions $\pi: \mathcal{X} \to [0,1]$, and suppose $X$ is a discrete random variable supported on $\mathcal{X}$. Then the agent's best response to $w_{\text{ATO}}$ is:
    \[\pi^{w_{\text{ATO}}}(x) = \begin{cases}
        1 &\text{if } x \in \arg\max_{x} \mu_1(x) \text{ and } \mu_1(x) > 0\\
        0 & \text{otherwise}.
    \end{cases}\]
\end{proposition}
\begin{proof}
By the tower property,
\[E[w_{\text{ATO}}(\mathbf{X}, \mathbf{T}^{\pi}, \mathbf{Y})] =  E[E[Y_i | T_i^{\pi} = 1, X_i] | T_i^{\pi} = 1] \\
        = \int E[Y_i | T_i^{\pi} = 1, X_i] dP_{X_i | T_i^{\pi} = 1}.\]
Since $T_i^{\pi}$ only depends on $X_i$ by construction, 
we have the \textit{strong ignorability} property that $\{Y_i(0), Y_i(1)\} \indep T_i^{\pi} | X_i$. Therefore,
\[\int E[Y_i | T_i^{\pi} = 1, X_i] dP_{X_i | T_i^{\pi} = 1} = \int \mu_1(X_i) dP_{X_i | T_i^{\pi} = 1}.\]
For discrete $X_i$, this is equal to   
\[\sum_{x \in \mathcal{X}} \mu_1(x) P(X_i = x | T_i^{\pi} = 1) = \sum_{x \in \mathcal{X}} \mu_1(x) \frac{\pi(x)P(X_i = x)}{\sum_{z \in \mathcal{X}}\pi(z) P(X_i = z)}.\]
The $\pi$ function that maximizes this is exactly that given in Proposition \ref{prop:ato_best_response}.
\end{proof}

As a tool to prove further results, we give the following lemma for the optimal treatment rule.

\begin{lemma}[Optimal Treatment Rule]\label{lem:optimal_pi}
    If the agent's treatment rule $\pi$ depends only on $X_i$, then the treatment rule $\pi^{*}$ that maximizes $V(\pi)$ is 
    \[\pi^{*}(x) = \begin{cases}
    1 & \text{if } \tau(x) > 0 \\
    0 & \text{otherwise}.
    \end{cases}\]
\end{lemma}
\begin{proof}
    \begin{align*}
        V(\pi) &= E[Y_i(T_i^{\pi}) - Y_i(0)] \\
        &= E[Y_i(T_i^{\pi}) - Y_i(0)| T_i^{\pi} = 1]P(T_i^{\pi} = 1) + E[Y_i(T_i^{\pi}) - Y_i(0)| T_i^{\pi} = 0]P(T_i^{\pi} = 0) \\
        &= E[Y_i(1) - Y_i(0)| T_i^{\pi} = 1]P(T_i^{\pi} = 1).
    \end{align*}
    By the tower property, 
        \[E[Y_i(1) - Y_i(0)| T_i^{\pi} = 1] = E[E[Y_i(1) - Y_i(0) | T_i^{\pi} = 1, X_i] | T_i^{\pi} = 1].\]
    Since we have ignorability in $X_i$, 
    \[E[Y_i(1) - Y_i(0) | T_i^{\pi} = 1, X_i] = E[Y_i(1) - Y_i(0) | X_i] = \tau(X_i).\]
    Thus, 
    \[V(\pi) = E[\tau(X_i)| T_i^{\pi} = 1]P(T_i^{\pi} = 1) = \int \tau(X_i) P(T_i^{\pi} = 1) dP_{X_i | T_i^{\pi} = 1}.\]
    By Bayes' theorem, 
    \[\int \tau(X_i) P(T_i^{\pi} = 1) dP_{X_i | T_i^{\pi} = 1} = \int \tau(X_i) P(T_i^{\pi} = 1 | X_i) dP_{X_i} = \int \tau(X_i) \pi(X_i) dP_{X_i}. \]
    The $\pi$ function that maximizes this is exactly as given in Lemma \ref{lem:optimal_pi}.
\end{proof}

\begin{proposition}[ATT Best Response]\label{prop:att_best_response}
    Suppose $\Pi$ is the set of all functions $\pi: \mathcal{X} \to [0,1]$, and suppose $X$ is a discrete random variable supported on $\mathcal{X}$. Then the agent's best response to $w_{\text{ATT}}$ is:
    \[\pi^{w_{\text{ATT}}}(x) = \begin{cases}
    1 & \text{if } x \in \arg\max_{x} \tau(x) \text{ and } \tau(x) > 0 \\
    0 & \text{otherwise}.
    \end{cases}\]
\end{proposition}
\begin{proof}
    \begin{align*}
        E[w_{\text{ATT}}(\mathbf{X}, \mathbf{T}^{\pi}, \mathbf{Y})] &=  E[Y_i - \hat{\mu}_0(X_i)| T_i^{\pi} = 1] \\
        &= E[Y_i | T_i^{\pi} = 1] - E[E[\hat{\mu}_0(X_i) | T_i^{\pi} = 1, X_i] | T_i^{\pi} = 1].
    \end{align*}
    Since $E[\hat{\mu}_0(x)] = \mu_0(x)$ for all $x \in \mathcal{X}$, we have $E[\hat{\mu}_0(X_i) | T_i^{\pi} = 1, X_i] = \mu_0(X_i)$. Thus, 
    \[E[w_{\text{ATT}}(\mathbf{X}, \mathbf{T}^{\pi}, \mathbf{Y})] = E[Y_i | T_i^{\pi} = 1] - E[\mu_0(X_i) | T_i^{\pi} = 1]. \]
    Since we have ignorability in $X_i$,
    \[E[Y_i | T_i^{\pi} = 1] = E[E[Y_i | T_i^{\pi} = 1, X_i] | T_i^{\pi} = 1] = E[E[Y_i(1) | X_i]| T_i^{\pi} = 1] = E[\mu_1(X_i) | T_i^{\pi} = 1].\]

    Combining these, \[E[w_{\text{ATT}}(\mathbf{X}, \mathbf{T}^{\pi}, \mathbf{Y})] = E[ \mu_1(X_i) - \mu_0(X_i) | T_i^{\pi} = 1] = E[\tau(X_i) | T_i^{\pi} = 1].\]

    For discrete $X_i$, this is equal to 
    \[\sum_{x \in \mathcal{X}} \tau(x) P(X_i = x | T_i^{\pi} = 1) = \sum_{x \in \mathcal{X}} \tau(x) \frac{\pi(x)P(X_i = x)}{\sum_{z \in \mathcal{X}}\pi(z) P(X_i = z)}.\]
    
    The $\pi$ function that maximizes this is exactly that given in Proposition \ref{prop:att_best_response}.
\end{proof}

\begin{proposition}[TO Best Response]\label{prop:to_best_response}
    Suppose $\Pi$ is the set of all functions $\pi: \mathcal{X} \to [0,1]$. Then the agent's best response to $w_{\text{TO}}$ is:
    \[\pi^{w_{\text{TO}}}(x) = \begin{cases}
    1 & \text{if } \mu_1(x) > 0 \\
    0 & \text{otherwise}.
    \end{cases}\]
\end{proposition}
\begin{proof}
    \[E[w_{\text{TO}}(\mathbf{X}, \mathbf{T}^{\pi}, \mathbf{Y})] = n E[Y_i | T_i^{\pi} = 1] P(T_i^{\pi} = 1).\]
    Since we have ignorability in $X_i$, we have $E[Y_i | T_i^{\pi} = 1] = E[\mu_1(X_i) | T_i^{\pi} = 1]$ (shown in more detail in the above proofs). Thus, 
    \[E[w_{\text{TO}}(\mathbf{X}, \mathbf{T}^{\pi}, \mathbf{Y})] = n E[\mu_1(X_i) | T_i^{\pi} = 1] P(T_i^{\pi} = 1) = n \int \mu_1(X_i) P(T_i^{\pi} = 1) dP_{X_i | T_i^{\pi} = 1}.\]
    By Bayes' theorem, 
    \[n\int \mu_1(X_i) P(T_i^{\pi} = 1) dP_{X_i | T_i^{\pi} = 1} = n\int \mu_1(X_i) P(T_i^{\pi} = 1 | X_i) dP_{X_i} = n\int \mu_1(X_i) \pi(X_i) dP_{X_i}. \]
    The $\pi$ function that maximizes this is exactly that given in Proposition \ref{prop:to_best_response}.
\end{proof}

\begin{proposition}[TT Best Response]\label{prop:tt_best_response}
    Suppose $\Pi$ is the set of all functions $\pi: \mathcal{X} \to [0,1]$. Then the agent's best response to $w_{\text{TT}}$ is:
    \[\pi^{w_{\text{TT}}}(x) = \begin{cases}
    1 & \text{if } \tau(x) > 0 \\
    0 & \text{otherwise}.
    \end{cases}\]
\end{proposition}
\begin{proof}
    \[E[w_{\text{TT}}(\mathbf{X}, \mathbf{T}^{\pi}, \mathbf{Y})] = n E[Y_i - \hat{\mu}_0(X_i) | T_i^{\pi} = 1] P(T_i^{\pi} = 1).\]
    Since we have ignorability in $X_i$, we have $E[Y_i | T_i^{\pi} = 1] = E[\mu_1(X_i) | T_i^{\pi} = 1]$.  (shown in more detail in the above proofs). 
    Since $E[\hat{\mu}_0(x)] = \mu_0(x)$ for all $x \in \mathcal{X}$, we have $E[\hat{\mu}_0(X_i) | T_i^{\pi} = 1] = E[\mu_0(X_i) | T_i^{\pi} = 1]$. Thus, 
    \[E[w_{\text{TT}}(\mathbf{X}, \mathbf{T}^{\pi}, \mathbf{Y})] = n E[\tau(X_i) | T_i^{\pi} = 1] P(T_i^{\pi} = 1) = n V(\pi).\]
    Applying Lemma \ref{lem:optimal_pi}, the $\pi$ function that maximizes this is exactly that given in Proposition \ref{prop:tt_best_response}.
\end{proof}

\subsection{Regret proofs}
We give proofs for the regret bounds stated for each reward function in Section \ref{sec:comparisons}.

\begin{repproposition}{prop:ato_regret}[ATO Regret]
    If the conditional mean untreated potential outcomes $\mu_0(x)$ are unbounded, then the regret for the reward function $w_{\text{ATO}}$ may be arbitrarily large.
\end{repproposition}
\begin{proof}
    We construct an family of distributions for which the regret is unbounded. Let $\mathcal{X} = \{0,1\}$ with $P(X_i = 1) = p$, and let $\mu_1(0) = 0$, $\mu_1(1) = 1$. Let $\mu_0(0) = \alpha$, and $\mu_0(1) = \beta$. Suppose $\Pi$ is the set of all functions $\pi: \mathcal{X} \to [0,1]$.

    For $w_{\text{ATO}}$, the agent's best response is $\pi^{w_{\text{ATO}}}(0) = 0$ and $\pi^{w_{\text{ATO}}}(1) = 1$.
    
    We illustrate two failure modes for the reward function $w_{\text{ATO}}$. First, suppose $\alpha < 0$, and $\beta = 0$. Then the regret is given by
    \[R(\pi^w) = \max_{\pi \in \Pi} V(\pi) - V(\pi^w) = -\alpha(1-p) + p - p = -\alpha(1-p). \]
    This is unbounded for unbounded $\alpha$. Intuitively, this illustrates an example where the agent ignores the higher treatment effect of the patients with a lower treated outcome, such as sicker patients with higher mortality probability but more benefit from surgery. Since the agent's best response does not account for the patient's untreated potential outcome, the agent thus ignores the fact that sicker patients would otherwise have very poor outcomes without treatment. This matches the findings from \citet{dranove2003more}.

    A second failure mode would be if $\alpha = -1$, and $\beta > 1$. Then the regret is given by
    \[R(\pi^w) = (1-p) - p(\beta - 1) = 1 - p\beta. \]
    This is also unbounded for unbounded $\beta$. This illustrates an example where the agent treats the patients with a higher treated outcome, even though treatment actually harms those patients, such as healthier patients who might incur more risks or side effects from treatment.
\end{proof}

\begin{repproposition}{prop:att_regret}[ATT Regret]
    Suppose $\Pi$ is the set of all functions $\pi: \mathcal{X} \to [0,1]$. Then the regret for the reward function $w_{\text{ATT}}$ is upper bounded by
    \[R(\pi^{w_{\text{ATT}}}) \leq \max_{\pi \in \Pi} V(\pi). \]
\end{repproposition}
\begin{proof}
    Let $\pi^{*} = \argmax_{\pi \in \Pi} V(\pi).$ Then
    \[\pi^{*}(x) = \begin{cases}
        1 & \text{if } \tau(x) > 0 \\
        0 & \text{otherwise}.
    \end{cases}\]
    \begin{align*}
        R(\pi^{w_{\text{ATT}}}) &= V(\pi^{*}) - V(\pi^{w_{\text{ATT}}}) \\
        &=E[\pi^{*}(X_i) \tau(X_i)] - E[\pi^{w_{\text{ATT}}}(X_i) \tau(X_i)] \\
        &=E[\tau(X_i) \Ind(\tau(X_i) > 0)] - E[\tau(X_i) \Ind(\tau(X_i) > 0 \cap X_i \in \arg\max_{x \in \mathcal{X}} \tau(x))] \\
        &= E[\tau(X_i)\Ind(\tau(X_i) > 0 \cap X_i \notin \arg\max_{x \in \mathcal{X}} \tau(x))]
    \end{align*}
\end{proof}

\begin{repproposition}{prop:tt_regret}[TT Regret]
    If $E[\hat{\mu}_0(X)] = \mu_0(X)$, then the regret from applying the reward function $w_{\text{TT}}$ is $R(\pi^{w_{\text{TT}}}) = 0$.
\end{repproposition}
\begin{proof}
    As shown in Proposition \ref{prop:tt_best_response}, \[E[w_{\text{TT}}(\mathbf{X}, \mathbf{T}^{\pi}, \mathbf{Y})] = n E[\tau(X_i) | T_i^{\pi} = 1] P(T_i^{\pi} = 1) = n V(\pi).\] Therefore, the agent's best response $\pi^{w_{\text{TT}}}$ also maximizes $V(\pi)$ over the same feasible set, $\Pi$.
\end{proof}

\section{Proofs from Section \ref{sec:ranking} on Ranking}

We give expanded notation and proofs for the results in Section \ref{sec:ranking} on modifications of $w_{\text{TT}}$ to satisfy ranking desiderata.

\subsection{Detailed notation for ranking with multiple agents}
We expand the notation for ranking with multiple agents, as applied in Section \ref{sec:ranking}. Suppose there are $K$ agents, where each agent $k$ observes its own sample of $n_k$ patients with covariates $\mathbf{X}^{(k)} = \{X_i^{(k)}\}_{i=1}^{n_k}$ drawn i.i.d.\ from distribution $P_{X^{(k)}}$ with support $\mathcal{X}$. Let $Y_i^{(k)}(t)$ denote the potential outcomes when agent $k$ treats the patient with treatment $t$. Let $\mu^{(k)}_t(x) = E[Y_i^{(k)}(t) | X^{(k)}_i = x]$, and $\tau^{(k)}(x) = E[Y_i^{(k)}(1) - Y_i^{(k)}(0) | X^{(k)}_i = x]$. 

For rankings to be meaningful, we assume that if the same patient with covariate value $x \in \mathcal{X}$ were to be treated by either agent $j$ or agent $k$, their potential outcomes would follow each agent's respective conditional potential outcome distributions for covariate value $x$. Furthermore, the conditional potential untreated outcome has the same distribution for all $k$: for each $x \in \mathcal{X}$, the distributions $P_{Y^{(k)}(0) | X^{(k)} = x}$ are identical for all $k \in \{1,...,K\}$. Let $\mu_0(x)$ denote the shared mean conditional untreated potential outcome, with $\mu^{(k)}_0(x) = \mu_0(x)$ for all $k$. In short, one provider not treating a patient is equivalent to another provider not treating the same patient.
 
Suppose agent $k$ chooses treatment policy $\pi_{k}$, thus producing realized treatments denoted $\mathbf{T}^{\pi_k} = \{T_i^{\pi_k}\}_{i=1}^{n_k}$ and outcomes $Y_i^{(k)} = Y_i^{(k)}(T_i^{\pi_k})$, $\mathbf{Y}^{(k)} = \{Y_i^{(k)}\}_{i=1}^{n_k}$.

Suppose the reward function $w$ is used to rank these $K$ agents in the following way: the principal publishes score functions $w_k: \mathcal{X}^{n_k} \times \{0,1\}^{n_k} \times \mathbb{R}^{n_k} \to \mathbb{R}$, and each agent $k$ gets score $w_k(\mathbf{X}^{(k)},\mathbf{T}^{\pi_k}, \mathbf{Y}^{(k)})$ after choosing their treatment policy $\pi_k$. The agents are then ranked from highest to lowest score function values.

We assume that each agent seeks to maximize their individual ranking, and make the simplifying assumption that the agents act independently: that is, each agent does not consider the potential actions of other agents when choosing their actions. This may be realistic in a setting with a large number of hospitals serving more-or-less independent populations, though more complex competitive multi-agent models may make for interesting future extensions.

\subsection{Proofs}

We give proofs for the proposition and theorems from Section \ref{sec:ranking}.

\begin{repproposition}{prop:tt_ranking_violated}
    If \[w_k(\mathbf{X}^{(k)}, \mathbf{T}^{\pi_k}, \mathbf{Y}^{(k)}) = \sum_{i=1}^{n_k} (Y_i^{(k)} - \hat{\mu}_0(X_i^{(k)})) T_i^{\pi_k},\]
    then both ranking properties in Definitions \ref{def:uniform_ranking} and \ref{def:relative_ranking} will be violated.
\end{repproposition}
\begin{proof}
    Suppose for agents $j$ and $k$, $\Pi_j$ and $\Pi_k$ are both unconstrained. For $w_k$ as defined above,
    \[\max_{\pi_j \in \Pi_j}E[w_j(\mathbf{X}^{(j)},\mathbf{T}^{\pi_j}, \mathbf{Y}^{(k)})] = n_j E[\tau^{(j)}(X^{(j)}) \Ind(\tau^{(j)}(X^{(j)}) > 0)], \] 
    \[\max_{\pi_k \in \Pi_k}E[w_k(\mathbf{X}^{(k)},\mathbf{T}^{\pi_k}, \mathbf{Y}^{(k)})] = n_k E[\tau^{(k)}(X^{(k)}) \Ind(\tau^{(k)}(X^{(k)}) > 0)]. \]
    
    For Definition \ref{def:uniform_ranking}, suppose $\tau^{(j)}(x) \geq \tau^{(k)}(x)$ for all $x \in \mathcal{X}$, and $X^{(j)}$ and $X^{(k)}$ are identically distributed. Then,
    \[E[\tau^{(j)}(X^{(j)}) \Ind(\tau^{(j)}(X^{(j)}) > 0)] \geq E[\tau^{(k)}(X^{(k)}) \Ind(\tau^{(k)}(X^{(k)}) > 0)].\]
     Let $n_j, n_k$ be a pair such that \[n_j < n_k \frac{E[\tau^{(k)}(X^{(k)}) \Ind(\tau^{(k)}(X^{(k)}) > 0)]}{E[\tau^{(j)}(X^{(j)}) \Ind(\tau^{(j)}(X^{(j)}) > 0)]}.\] 
    This immediately results in Definition \ref{def:uniform_ranking} being violated.

    Definition \ref{def:relative_ranking} is also violated, since $\tau^{(j)}(x) \geq \tau^{(k)}(x)$ for all $x \in \mathcal{X}$ implies that $E[\tau^{(j)}(X_0)] \geq E[\tau^{(k)}(X_0)]$  for any reference population $P_{X_0}$. However, for the above $n_j, n_k$, we've shown that \[\max_{\pi_j \in \Pi_j}E[w_j(\mathbf{X}^{(j)},\mathbf{T}^{\pi_j}, \mathbf{Y}^{(k)})] < \max_{\pi_k \in \Pi_k}E[w_k(\mathbf{X}^{(k)},\mathbf{T}^{\pi_k}, \mathbf{Y}^{(k)})].\]
\end{proof}

\begin{reptheorem}{thm:reweighting}[Incentive Alignment]
    Suppose $E[\hat{\mu}_0(x)] = \mu_0(x)$, and suppose $\Pi$ is the set of all functions $\pi: \mathcal{X} \to [0,1]$. For any function $g: \mathcal{X} \to \mathbb{R}^+$, $w_{\text{TT}}^g$ yields an agent best response with zero regret. 
\end{reptheorem}

\begin{proof}
By the tower property, 
\begin{align*}
    E[w_{\text{TT}}^g(\mathbf{Y}, \mathbf{T}^{\pi}, \mathbf{X})] &= nP(T^{\pi}_i = 1)E[E[(Y_i - \hat{\mu}_0(X_i))g(X_i)| T_i^{\pi} = 1, X_i]| T_i^{\pi} = 1]\\
    &= nP(T^{\pi}_i = 1)E[g(X_i) E[Y_i - \hat{\mu}_0(X_i)| T_i^{\pi} = 1, X_i]| T_i^{\pi} = 1].
\end{align*}


Since $E[\hat{\mu}_0(x)] = \mu_0(x)$ and we have ignorability in $X_i$, 
\[E[Y_i - \hat{\mu}_0(X_i)| T_i^{\pi} = 1, X_i] = \tau(X_i). \]
This is shown in more detail in the proof for Proposition \ref{prop:tt_best_response}. 

Combining this with the above, we have
\[E[w_{\text{TT}}^g(\mathbf{Y}, \mathbf{T}^{\pi}, \mathbf{X})] = nP(T^{\pi}_i = 1)E[g(X_i) \tau(X_i)| T_i^{\pi} = 1]. \]

Applying Bayes' theorem as in Lemma \ref{lem:optimal_pi},
\begin{align*}
    P(T^{\pi}_i = 1)E[g(X_i) \tau(X_i)| T_i^{\pi} = 1] &= \int g(X_i) \tau(X_i) P(T^{\pi}_i = 1) dP_{X_i | T^{\pi}_i = 1} \\
    &= \int g(X_i) \tau(X_i) P(T^{\pi}_i = 1 | X_i) dP_{X_i}\\
    &= \int g(X_i) \tau(X_i) \pi(X_i) dP_{X_i}.
\end{align*}
Therefore, 
\[E[w_{\text{TT}}^g(\mathbf{Y}, \mathbf{T}^{\pi}, \mathbf{X})] = nE[g(X_i) \tau(X_i) \pi(X_i)].\]

Since $g(X_i) > 0$, the treatment rule $\pi$ that maximizes this is the same as $\pi^{*}$ from Lemma \ref{lem:optimal_pi}. Therefore, the regret is zero.
\end{proof}

\begin{reptheorem}{thm:ranking_reweighted}[Ranking Desiderata Satisfied]
    Let $P_{X^{(k)}}$ be absolutely continuous with respect to $P_{X_0}$, and let $g_k = \frac{1}{n_k}\frac{dP_{X_0}}{dP_{X^{(k)}}}$ be the normalized Radon–Nikodym derivative of the reference distribution $P_{X_0}$ with respect to agent $k$'s covariate distribution $P_{X^{(k)}}$. Then 
    \begin{equation}\label{eq:ranking_reweighted}
        w_k(\mathbf{X}^{(k)}, \mathbf{T}^{\pi_k}, \mathbf{Y}^{(k)}) = \sum_{i=1}^{n_k} (Y_i^{(k)} - \hat{\mu}_0(X_i^{(k)})) T_i^{\pi_k} g_k(X^{(k)})
    \end{equation}
    satisfies both ranking properties in Definitions \ref{def:uniform_ranking} and \ref{def:relative_ranking} as long as $\Pi_k$ is unconstrained and treatment effects are nonnegative, $\tau^{(k)}(x) \geq 0$, for all $k \in \{1,...,K\}$.
\end{reptheorem}
\begin{proof}
    As shown in the proof of Theorem \ref{thm:reweighting},
    \[E[w_k(\mathbf{X}^{(k)}, \mathbf{T}^{\pi_k}, \mathbf{Y}^{(k)})] = n_k E[\tau^{(k)}(X_i^{(k)})\pi_k(X_i^{(k)}) g_k(X_i^{(k)})].\]
    
    With $g_k = \frac{1}{n_k}\frac{dP_{X_0}}{dP_{X^{(k)}}}$, we have
    \[n_k E[\tau^{(k)}(X_i^{(k)})\pi_k(X_i^{(k)}) g_k(X_i^{(k)})] = E[\tau^{(k)}(X_0)\pi_k(X_0)].\]

    Let $\Pi_j$, $\Pi_k$ both be unconstrained. Then 
    \[\max_{\pi_j \in \Pi_j}E[w_j(\mathbf{X}^{(j)},\mathbf{T}^{\pi_j}, \mathbf{Y}^{(k)})] = E[\tau^{(j)}(X_0) \Ind(\tau^{(j)}(X_0) > 0)], \] 
    \[\max_{\pi_k \in \Pi_k}E[w_k(\mathbf{X}^{(k)},\mathbf{T}^{\pi_k}, \mathbf{Y}^{(k)})] = E[\tau^{(k)}(X_0) \Ind(\tau^{(k)}(X_0) > 0)]. \]
    Definition \ref{def:uniform_ranking} is immediately satisfied, since $\tau^{(j)}(x) \geq \tau^{(k)}(x)$ for all $x \in \mathcal{X}$ implies $E[\tau^{(k)}(X_0) \Ind(\tau^{(k)}(X_0) > 0)] \geq E[\tau^{(k)}(X_0) \Ind(\tau^{(k)}(X_0) > 0)]$.
    
    If $\tau^{(k)}(x) \geq 0$ for all $x \in \mathcal{X}$ and all $k \in \{1,...,K\}$, then 
    Definition \ref{def:relative_ranking} is satisfied, since 
    \[E[\tau^{(j)}(X_0)] \geq E[\tau^{(k)}(X_0)] \implies E[\tau^{(j)}(X_0) \Ind(\tau^{(j)}(X_0) > 0)] \geq E[\tau^{(k)}(X_0) \Ind(\tau^{(k)}(X_0) > 0)].\]
\end{proof}

\section{Proofs and Additional Results for Section \ref{sec:info_asymmetry} on Information Asymmetry}
In this section, we give proofs and additional regret bound results for Section \ref{sec:info_asymmetry} on information asymmetry. First, we prove the regret bounds in Section \ref{sec:info_asymmetry} when $\hat{\mu}_0(x)$ is estimated from auxiliary data where the mean untreated potential outcome conditioned on $X$, $\mu_0(x)$, is identifiable.

Second, we prove similar results when $\hat{\mu}_0(x)$ is estimated from the untreated units in the treatment population, where $T^{\pi} = 0$. Under information asymmetry, the mean untreated potential outcome conditioned on $X$, $\mu_0(X)$, is no longer identifiable. Therefore, $\hat{\mu}_0(x)$ will be subject to confounding bias. Still, if we apply a similar but stronger assumption than Assumption \ref{assumption:unobserved-confounding-avg-historical}, we can get similar regret bounds to Theorem \ref{thm:historical_regret}. 

\subsection{Detailed notation for information asymmetry}

We model information asymmetry in our principal agent game as follows: suppose the agent observes additional covariates per patient $U_i \in \mathcal{U}$, and selects a treatment rule $\pi: \mathcal{X}, \mathcal{U} \to [0,1]$ from a feasible set of treatment rules $\Pi$, with $\pi(X, U) = P(T_i^{\pi} = 1 | X, U)$. Suppose the principal still observes only $\{X_i, T_i^{\pi}, Y_i\}_{i=1}^n$, and chooses a reward function $w: \mathcal{X}^n \times \{0,1\}^n \times \mathbb{R}^n \to \mathbb{R}$ with which to reward the agent. Notably, the principal's reward function $w$ cannot depend on $U$. Let $\mu_t(X,U) = E[Y_i(t) | X,U]$ and $\tau(X,U) = E[Y_i(1) - Y_i(0) | X, U]$.

The utility is still defined as in Section \ref{sec:model}, and with the additional $U$ variable can be rewritten as \[V(\pi) = E[\pi(X_i,U_i)\tau(X_i,U_i)].\]
The regret is also still defined as in Section \ref{sec:model}.

\subsection{Proofs for regret bounds with unbiased counterfactual estimate}

Suppose the principal estimates the mean conditional untreated potential outcome from auxiliary data $\{X_j', T_j', Y_j'\}_{j=1}^m$ drawn i.i.d.\ from auxiliary dataset $\mathcal{Q}$, denoted $\hat{\mu}^{\mathcal{Q}}_0(x)$. Suppose we have a ``best case scenario'' where the relationship between $X_j'$ and $Y_j'(0)$ is the same in the auxiliary data as the relationship between $X_i$ and $Y_i(0)$ in the treatment population, and the mean untreated potential outcome conditional on $X_j'$ is identifiable from $\mathcal{Q}$, such that \[E[\hat{\mu}^{\mathcal{Q}}_0(x)] = E[Y_j'(0) | X_j' = x] = E[Y_i(0) | X_i = x].\]

Outside of this ideal setting, any distribution shift or problems with identifiability in the $\mathcal{Q}$ dataset would increase the regret. Proposition \ref{prop:regret_to_bias} below is an intermediate result that does not actually rely on identifiability of $\mu_0(x)$, and may provide a good starting point for future analyses of distribution shift or non-identifiability. To simplify the proofs below, we first give Lemma \ref{lem:info_asymmetry_w}.

\begin{lemma}\label{lem:info_asymmetry_w} Under information asymmetry,
    \[E[w_{\text{TT}}(\mathbf{X}, \mathbf{T}^{\pi}, \mathbf{Y})] = n(E[\pi(X_i,U_i) \mu_1(X_i, U_i)] - E[\pi(X_i, U_i) \hat{\mu}_0(X_i)]).\]
\end{lemma}
\begin{proof}
    We have previously shown that 
    \begin{align*}
        E[w_{\text{TT}}(\mathbf{X}, \mathbf{T}^{\pi}, \mathbf{Y})] &= nP(T_i^{\pi} = 1)E[Y_i - \hat{\mu}_0(X_i) | T_i^{\pi} = 1] \\
        &= n(P(T_i^{\pi} = 1)E[Y_i| T_i^{\pi} = 1] - P(T_i^{\pi} = 1)E[\hat{\mu}_0(X_i) | T_i^{\pi} = 1]).
    \end{align*}
    Considering the first term, since we have ignorability in $X_i$ and $U_i$,
    \begin{align*}
        P(T_i^{\pi} = 1)E[Y_i| T_i^{\pi} = 1] &= P(T_i^{\pi} = 1)E[E[Y_i | X_i, U_i, T_i^{\pi} = 1]| T_i^{\pi} = 1] \\
        &= P(T_i^{\pi} = 1)E[\mu_1(X_i, U_i)| T_i^{\pi} = 1] \\
        &= P(T_i^{\pi} = 1) \int \mu_1(X_i, U_i) dP_{X_i, U_i | T_i^{\pi} = 1} \\
        &= \int \mu_1(X_i, U_i) P(T_i^{\pi} = 1 | X_i, U_i)dP_{X_i, U_i} \\
        &= E[\pi(X_i, U_i)\mu_1(X_i, U_i) ].
    \end{align*}
    For the second term, we again apply Bayes' theorem:
    \begin{align*}
        P(T_i^{\pi} = 1)E[\hat{\mu}_0(X_i) | T_i^{\pi} = 1] &= P(T_i^{\pi} = 1) \int \hat{\mu}_0(X_i) dP_{X_i, U_i | T_i^{\pi} = 1} \\
        &= \int \hat{\mu}_0(X_i) P(T_i^{\pi} = 1 | X_i, U_i)dP_{X_i, U_i} \\
        &= E[\pi(X_i, U_i)\hat{\mu}_0(X_i)].
    \end{align*}
\end{proof}


\begin{proposition}\label{prop:regret_to_bias}
Suppose the principal applies the reward function $w_{\text{TT}}$ with an estimate $\hat{\mu}^{\mathcal{Q}}_0(X)$. Then the regret is bounded by the average bias in the conditional untreated potential outcome estimate.
\begin{equation}
    R(\pi^w) \leq 2 E[|\hat{\mu}^{\mathcal{Q}}_0(X_i) - \mu_0(X_i, U_i)| ].
\end{equation}
\end{proposition}
\begin{proof}
Let $\hat{V}(\pi) = \frac{1}{n}E[{w_{\text{TT}}}(\mathbf{X}, \mathbf{T}^{\pi}, \mathbf{Y})]$. Then $\pi^{w_{\text{TT}}}$ maximizes $\hat{V}(\pi)$ as well. 
\begin{align*}
    V(\pi^{*}) - V(\pi^{w_{\text{TT}}}) &\leq V(\pi^{*}) - V(\pi^{w_{\text{TT}}}) + \hat{V}(\pi^{w_{\text{TT}}}) - \hat{V}(\pi^{*}) \\
    &\leq |V(\pi^{*}) - \hat{V}(\pi^{*})| + |\hat{V}(\pi^{w_{\text{TT}}}) - V(\pi^{w_{\text{TT}}})| \\
    &\leq 2 \max_{\pi \in \Pi}|\hat{V}(\pi) - V(\pi)| \\
    &= 2 \max_{\pi \in \Pi}|E[\pi(X_i,U_i)(\hat{\mu}^{\mathcal{Q}}_0(X_i) - \mu_0(X_i, U_i))]|,
\end{align*}
where the last line follows from Lemma \ref{lem:info_asymmetry_w}.
By Jensen's inequality,
\begin{align*}
    \max_{\pi \in \Pi}|E[\pi(X,U)(\hat{\mu}_0^{\mathcal{Q}}(X_i) - \mu_0(X_i, U_i))]|  &\leq \max_{\pi \in \Pi}E[\pi(X_i,U_i)|\hat{\mu}_0^{\mathcal{Q}}(X_i) - \mu_0(X_i, U_i)| ] \\
    &\leq E[|\hat{\mu}_0^{\mathcal{Q}}(X_i) - \mu_0(X_i, U_i)| ].
\end{align*}
\end{proof}

\begin{reptheorem}{thm:historical_regret}[Regret With Information Asymmetry]
    If the principal applies the reward function from $w_{\text{TT}}$ with an unbiased estimate $\hat{\mu}_0(X)$ where $E[\hat{\mu}_0(x)] = \mu_0(x)$, then under Assumption \ref{assumption:unobserved-confounding-avg-historical}, the regret is upper bounded as
    \[R(\pi^{w_{\text{TT}}}) \leq 2\gamma_{\text{marg}}.\]
\end{reptheorem}
\begin{proof}
    This follows similarly to the proof of Proposition \ref{prop:regret_to_bias}:
    \begin{align*}
    V(\pi^{*}) - V(\pi^{w_{\text{TT}}}) &\leq V(\pi^{*}) - V(\pi^{w_{\text{TT}}}) + \hat{V}(\pi^{w_{\text{TT}}}) - \hat{V}(\pi^{*}) \\
    &\leq |V(\pi^{*}) - \hat{V}(\pi^{*})| + |\hat{V}(\pi^{w_{\text{TT}}}) - V(\pi^{w_{\text{TT}}})| \\
    &\leq 2 \max_{\pi \in \Pi}|\hat{V}(\pi) - V(\pi)| \\
    &= 2 \max_{\pi \in \Pi}|E[\pi(X_i,U_i)(\hat{\mu}_0(X_i) - \mu_0(X_i, U_i))]|
\end{align*}
By the tower property, 
\begin{align*}
    E[\pi(X,U)\hat{\mu}_0(X)] &= E[E[\hat{\mu}_0(X_i)\pi(X_i,U_i)| X_i, U_i]] \\
    &= E[\pi(X_i,U_i)E[\hat{\mu}_0(X_i)| X_i, U_i]] \\
    &= E[\pi(X_i,U_i)\mu_0(X_i)].
\end{align*}
Therefore, 
\[E[\pi(X_i,U_i)(\hat{\mu}_0(X_i) - \mu_0(X_i, U_i))] = E[\pi(X_i,U_i)(\mu_0(X_i) - \mu_0(X_i, U_i))].\]
By Jensen's inequality,
\begin{align*}
    \max_{\pi \in \Pi}|E[\pi(X_i,U_i)(\mu_0(X_i) - \mu_0(X_i, U_i))]|  &\leq \max_{\pi \in \Pi}E[\pi(X_i,U_i)|\mu_0(X_i) - \mu_0(X_i, U_i)| ] \\
    &\leq E[|\mu_0(X_i) - \mu_0(X_i, U_i)| ].
\end{align*}
    The regret bound then follows directly from Assumption \ref{assumption:unobserved-confounding-avg-historical} which says that \[E[ | \mu_0(X_i) - \mu_0(X_i,U_i)| ] \leq \gamma_{\text{marg}}.\]
\end{proof}

\begin{repproposition}{prop:historical_regret_tight}
    For all $\epsilon > 0$, there exists a distribution of $X_i, U_i, Y_i(0), Y_i(1)$ 
    wherein $R(\pi^{w_{\text{TT}}}) \geq \gamma_{\text{marg}} -\epsilon$.
\end{repproposition}
\begin{proof}
    We construct a family of distributions of $X_i, U_i, Y_i(0), Y_i(1)$ that achieves this regret bound. Let $U \in \{0,1\}$, with $P(U = 1) = \frac{1}{2}$. Suppose $X$ is entirely uncorrelated with $Y_i(0), Y_i(1)$, such that $E[Y(t) | X] = E[Y(t)] = \mu_t$.
    For $\alpha > 0, \beta > 0$, let 
    \[\mu_1(x, u) = \begin{cases}
        0 & \text{if } u = 1 \\
        \beta & \text{if } u = 0
    \end{cases}, \quad \mu_0(x, u) = \begin{cases}
        -\alpha & \text{if } u = 1 \\
        \alpha & \text{if } u = 0
    \end{cases},\]
    which also means that $\mu_0(x) = \mu_0 = 0$ for all $x$.
    Suppose $\Pi$ includes all functions $\pi: \mathcal{X}, \mathcal{U} \to [0,1]$. Then by Lemma \ref{lem:info_asymmetry_w}, and assuming $\hat{\mu}_0(x)$ is unbiased, we have 
    \[E[w_{\text{TT}}(\mathbf{X}, \mathbf{T}^{\pi}, \mathbf{Y})] = n(E[\pi(X_i,U_i) (\mu_1(X_i, U_i)- \mu_0(X_i))].\]
    
    The resulting policy $\pi^{w_{\text{TT}}}$ that maximizes this is 
    \[\pi^{w_{\text{TT}}}(x,u) = \begin{cases}
        0 & \text{if } u = 1 \\
        1 & \text{if } u = 0
    \end{cases}.\]

    Suppose $\beta < \alpha$. Then the optimal policy is 
    \[\pi^{*}(x,u) = \begin{cases}
        1 & \text{if } u = 1 \\
        0 & \text{if } u = 0
    \end{cases}.\]
    with utility $V(\pi^{*}) = \frac{\alpha}{2}$. The regret $R(\pi^{w_{\text{TT}}})$ is then
    \[R(\pi^{w_{\text{TT}}}) = \frac{\alpha}{2} - \frac{\beta - \alpha}{2} = \alpha - \frac{\beta}{2}.\]
    Note that $\gamma_{\text{marg}} = \alpha$. For any $\epsilon$, choosing $\beta = \epsilon$ gets $R(\pi^{w_{\text{TT}}}) = \gamma_{\text{marg}} - \frac{\epsilon}{2}$, satisfying the bound in Proposition \ref{prop:historical_regret_tight}.
\end{proof}

\subsection{Additional regret bounds when estimating the counterfactual using agent data}

Suppose the principal estimates $\hat{\mu}_0(x)$ from the untreated data from the agent's treatment population, i.e. those individuals for whom $T_i^{\pi} = 0$. As discussed in Sections \ref{sec:info_asymmetry} and \ref{sec:estimating_mu_0}, under full information symmetry, the mean conditional untreated potential outcome is identifiable as long as the agent's treatment rule class $\Pi$ maintains positivity, $\pi(x) > 0$ for all $x \in \mathcal{X}$. Under information asymmetry, positivity is no longer sufficient for identifiability, as the agent's additional information makes $T_i^{\pi}$ depend on $U_i$, and thus ignorability in $X_i$ is no longer satisfied. 

Still, we can analyze what happens when the principal constructs a counterfactual estimator from the agent's untreated outcomes. This estimator depends on the agent's treatment rule $\pi$, and unlike in Section \ref{sec:estimating_mu_0}, information asymmetry means that the agent's treatment rule $\pi$ affects this estimator as well. Let $\hat{\mu}_0^{\pi}(X) = E[Y_i | T^{\pi} = 0, X]$ denote the principal's estimate of the mean untreated potential outcome from the agent's data.

As with the auxiliary data, if we assume a bound on how much $U_i$ affects the untreated potential outcome given $X_i$, we can still bound the regret if the principal were to apply the reward function $w_{\text{TT}}$ using $\hat{\mu}_0^{\pi}(X)$. However, the required assumption is a bit stronger.

\begin{assumption}\label{assumption:unobserved-confounding-avg} The maximum effect of the unobserved attributes $U$ on the conditional untreated potential outcome is bounded on average. Define the maximum difference in the untreated potential outcome for a given $x, u$ to be $$\Delta(x, u) = \max_{\tilde{u} \in \mathcal{U}} | \mu_0(x, u) - \mu_0(x, \tilde{u}) |.$$

The average difference is bounded as:
$$E[\Delta(X,U)] \leq \gamma_{\text{max}}. $$
\end{assumption}

Assumption \ref{assumption:unobserved-confounding-avg} is stronger than Assumption \ref{assumption:unobserved-confounding-avg-historical} in the sense that Assumption \ref{assumption:unobserved-confounding-avg-historical} is not sufficient to bound the regret when the principal estimates the mean conditional untreated potential outcome from the agent's data using $\hat{\mu}_0^{\pi}(X)$. Furthermore, the bound in Assumption \ref{assumption:unobserved-confounding-avg} implies that Assumption \ref{assumption:unobserved-confounding-avg-historical} is satisfied with $\gamma_{\text{marg}} \leq \gamma_{\text{max}}$. Intuitively, Assumption \ref{assumption:unobserved-confounding-avg} is not sufficient to bound the regret when the principal uses $\hat{\mu}_0^{\pi}(X)$ because the agent can choose $\pi$ to allocate treatment to single values of $\tilde{u} \in \mathcal{U}$, such that $\hat{\mu}_0^{\pi}(x)$ ends up matching a single value $\mu_0(x, \tilde{u})$.

More generally, Assumption \ref{assumption:unobserved-confounding-avg} is closer to bounds from senstivity analysis on expected outcome functions under unobserved confounding \cite{kennedy2022semiparametric}. While many existing sensitivity analyses bound the effect of unobserved confounding on treatment in prior data \cite{yadlowsky2022bounds}, in this case the agent's simultaneous treatment selection with data collection makes it less reasonable to bound the dependence of the treatment on $U$.

\begin{theorem}\label{thm:agent_regret}
    Suppose for all $\pi \in \Pi$, $\pi(x,u) > 0$ for all $x \in \mathcal{X}$, $u \in \mathcal{U}$. 
    If the principal applies the reward function $w_{\text{TT}}$ using an estimate $\hat{\mu}^{\pi}_0(x)$, then under Assumption \ref{assumption:unobserved-confounding-avg}, the regret is upper bounded as 
    \[R(\pi^w) \leq 2\gamma_{\text{max}}.\]
\end{theorem}
\begin{proof}
    This also follows similarly to the proof of Proposition \ref{prop:regret_to_bias}:
    \begin{align*}
    V(\pi^{*}) - V(\pi^{w_{\text{TT}}}) &\leq V(\pi^{*}) - V(\pi^{w_{\text{TT}}}) + \hat{V}(\pi^{w_{\text{TT}}}) - \hat{V}(\pi^{*}) \\
    &\leq |V(\pi^{*}) - \hat{V}(\pi^{*})| + |\hat{V}(\pi^{w_{\text{TT}}}) - V(\pi^{w_{\text{TT}}})| \\
    &\leq 2 \max_{\pi \in \Pi}|\hat{V}(\pi) - V(\pi)| \\
    &= 2 \max_{\pi \in \Pi}|E[\pi(X_i,U_i)(\hat{\mu}^{\pi}_0(X_i) - \mu_0(X_i, U_i))]| 
\end{align*}

By Jensen's inequality,
\begin{align*}
    \max_{\pi \in \Pi}|E[\pi(X,U)(\hat{\mu}_0^{\pi}(X) - \mu_0(X, U))]|  &\leq \max_{\pi \in \Pi}E[\pi(X,U)|\hat{\mu}_0^{\pi}(X) - \mu_0(X, U)| ] \\
    &\leq \max_{\pi \in \Pi}E[|\hat{\mu}_0^{\pi}(X) - \mu_0(X, U)| ]
\end{align*}

Define $$\Delta^{\Pi}(x, u) = \max_{\pi \in \Pi} |\hat{\mu}_0^{\pi}(x) - \mu_0(x, u)|.$$

Since for all $\pi \in \Pi$, $\pi(x,u) > 0$ for all $x \in \mathcal{X}$, $u \in \mathcal{U}$, we can apply Lemma \ref{lem:mu_0_subset} below, which says that $\hat{\mu}_0^{\pi}(x) \in [\min_u{\mu_0(x, u)}, \max_u{\mu_0(x, u)}]$.
Therefore, for all $x, u$, $$\Delta^{\Pi}(x, u) \leq \Delta(x, u).$$

Putting this together,
\[
    \max_{\pi \in \Pi}E[|\hat{\mu}_0^{\pi}(X) - \mu_0(X, U)| ] \leq E[\max_{\pi \in \Pi}|\hat{\mu}_0^{\pi}(X) - \mu_0(X, U)| ] \leq E[\Delta^{\Pi}(X, U)] \leq E[\Delta(X, U)].\]
Therefore, under assumption \ref{assumption:unobserved-confounding-avg}, $R(\pi^w) \leq 2\gamma_{\text{max}}$.
\end{proof} 

\begin{lemma}\label{lem:mu_0_subset} Suppose that for all $\pi \in \Pi$, $P_{U | X = x, T^{\pi} = 0}$ is a well defined probability distribution. Let $\hat{\mu}_0^{\pi}(x) = E[Y | X = x, T^{\pi} = 0]$. Then for all $\pi$, $\hat{\mu}_0^{\pi}(x) \in [\min_u{\mu_0(x, u)}, \max_u{\mu_0(x, u)}]$.
\end{lemma}
\begin{proof}
We decompose $\hat{\mu}_0^{\pi}(x)$ as 
$$\hat{\mu}_0^{\pi}(x) = E[Y | X = x, T^{\pi} = 0] = \int \mu_0(x,u) dP_{U | X = x, T^{\pi} = 0}$$

For any $\pi$, $\int \mu_0(x,u) dP_{U | X = x, T^{\pi} = 0} = \int \mu_0(x,u) dW(u)$ for some distribution $W$; therefore, 
\begin{align*}
    &\left\{\int \mu_0(x,u) dP_{U | X = x, T^{\pi} = 0}: \pi \in \Pi \right\} \\
    &\quad\quad\quad \subseteq \left\{\int \mu_0(x,u) dW(u) : W \text{ is a distribution over } U, \int dW(u) = 1 \right\}.
\end{align*}
Both the maximizer and the minimizer of the smaller set are contained in the larger set. Specifically,
\begin{align*}
    \max_{\pi}\hat{\mu}_0^{\pi}(x) &= \max_{\pi} \int \mu_0(x,u) dP_{U | X = x, T^{\pi} = 0} \leq \max_{W: \int dW(u) = 1} \int \mu_0(x,u) dW(u) \leq \max_u \mu_0(x, u) \\
    \min_{\pi}\hat{\mu}_0^{\pi}(x) &= \min_{\pi} \int \mu_0(x,u) dP_{U | X = x, T^{\pi} = 0} \geq \min_{W: \int dW(u) = 1} \int \mu_0(x,u) dW(u) \geq \min_u \mu_0(x, u)
\end{align*}
Therefore, for all $\pi$, $\hat{\mu}_0^{\pi}(x) \in [\min_u{\mu_0(x, u)}, \max_u{\mu_0(x, u)}]$.
\end{proof}



Overall, we have discussed two plausible data collection settings for the principal to estimate $\mu_0(x)$ to implement the reward function $w_{\text{TT}}$. In the first setting using auxiliary data where $\mu_0(x)$ is identifiable, the regret is bounded by the gap between the conditional effects of $U$ on $Y_i(0)$ and the marginal effect of only $X$ on $Y_i(0)$. In the second setting using untreated units from the agent's treatment population, $\mu_0(x)$ is not identifiable, and the regret can be bounded under a stronger assumption on the sensitivity of $\mu_0(X,U)$ to the agent's private information $U_i$. 

\section{Additional Experiment Details and Results}

We give additional experiment details and results here. 

\subsection{Implementation}

All models were trained using the linear regression and logistic regression packages from scikit-learn\cite{scikit-learn}.
All categorical features were one-hot encoded.

\subsection{Additional dataset details}

We provide additional setup and training details for each dataset.

\subsubsection{Horse Colic dataset}

For the Horse Colic dataset from UCI \cite{Dua:2019}, we let $X$ consist of all features observed prior to surgery, which includes 13 categorical features and 7 numerical features. All features used are listed in Figure \ref{fig:hc_gammas_regrets_cumul}. We removed all examples in which the horse was euthanized, and used the remaining ``outcome'' variable as $Y$, where we set $Y_i = 1$ if the horse lived, and $Y_i = -1$ if the horse died. We only used the main \texttt{horse-colic.data} dataset, and did not use the ``test'' dataset included in the UCI directory. 

To simulate $\mu_t(x)$, we assume that the outcome distribution takes the parametric form,
\[P(Y(t) = 1 | X = x) = \sigma(\beta_0 + \beta_1^\top x + \beta_2 t + \beta_3^\top xt)\]
where $\sigma(x) = \frac{1}{1 + e^{-x}}$ is the standard logistic function.

We trained a logistic regression model of $Y$ on $X$, $T$, and the interaction term $XT$ to estimate parameters $\hat{\beta}$, and used the resulting estimate to compute 
$\mu_t(x) = \sigma(\hat{\beta}_0 + \hat{\beta}_1^\top x + \hat{\beta}_2 t + \hat{\beta}_3^\top xt)$. For the fitted model, the AUC was $0.9924$, and the accuracy was $0.6824$. When computing agent best responses and regret, we take this function $\mu_t(x)$ to be given as synthetic mean conditional potential outcomes. Note that the clinical validity of $\mu_t(x)$ as actual potential outcomes cannot be verified from the dataset alone. There may be error in both our unconfoundedness assumption with respect to X, and in the logistic parametric specification of the outcome model.

In Table \ref{tab:regrets_full_info}, the ``demographic information'' for the Horse Colic dataset consists of only the \textit{age} feature.

\subsubsection{International Stroke Trial dataset}

For the International Stroke Trial dataset \cite{international1997international}, we let $X$ include all clinical data prior to treatment, which includes all ``Randomisation data'' except for dates and times. This includes $17$ categorical features and $3$ numerical features. All features used are listed in Figure \ref{fig:ist_gammas_regrets_cumul}. For the outcome variable $Y$, we apply the negative of the scalarized composite outcome score from \citet{kallus2018confounding}. Specifically,
\begin{align*}
    Y = &-2\Ind(\text{death}) - \Ind(\text{recurrent stroke})-0.5\Ind(\text{pulmonary embolism or intracranial bleeding}) \\
&-0.5\Ind(\text{other side effects}) + 2\Ind(\text{full recovery at 6 months}) + \Ind(\text{discharge within 14 days}).
\end{align*}
This results in $Y \in [-4, 3]$.

To simulate $\mu_t(x)$, we assume the conditional mean of the potential outcome distribution has the linear parametric form,
\[E[Y(t) | X = x] = \beta_0 + \beta_1^{\top} x + \beta_2 t + \beta_3^{\top} xt.\]

We trained a linear regression model of $Y$ on $X$, $T$, and the interaction term $XT$ to estimate parameters $\hat{\beta}$, and used the resulting estimate to compute $\mu_t(x) = \hat{\beta}_0 + \hat{\beta}_1^\top x + \hat{\beta}_2 t + \hat{\beta}_3^\top xt$. The fitted RMSE was $1.34$ and the $R^{2}$ was $0.26$. As with the Horse Colic dataset, when computing agent best responses and regret, we take this function $\mu_t(x)$ to be given as synthetic mean conditional potential outcomes. In this dataset the unconfoundedness assumption should hold as this data came from a randomized control trial. However, there may still be error in the linear parametric specification of the outcome model.

In Table \ref{tab:regrets_full_info}, the ``demographic information'' for the Stroke Trial dataset consists of both the \textit{age} feature and the \textit{sex} feature.

\subsection{Effect of better estimates of $\mu_t(x)$}

Better estimates of $\mu_t(x)$ may come from doubly robust estimators. However, the contribution of this work is not to improve estimation methods for $\mu_t(x)$, nor do the proposed reward policies rely on the quality or variance of $\mu_t(x)$ estimators. The quality of $\mu_t(x)$ only affects the how well the regrets reported in experiments might approximate real regrets. This is because in experiments, we estimate $\mu_t(x)$ to simulate potential outcomes for two purposes: \textit{(i)} simulating the agent's perception of potential outcomes and thus yielding the agent's best responses; and \textit{(ii)} calculating an approximate regret. Better estimates of $\mu_t(x)$ would improve the alignment with both of these simulations with reality. 

Our theory only assumes that the principal has an estimator $\hat{\mu}_0(x)$ which is unbiased with respect to the agent's preceived $\mu_0(x)$ function which the agent uses to calculate their best response. For the experiments, this unbiasedness assumption is true, as we assume both principal and agent have access to the same synthetic $\mu_t(x)$ values. We do not directly simulate the principal's calculation of $\hat{\mu}_0(x)$ from a subset of the data, since this would not affect the actions of the expectation maximizing agent.

The main limitation of our experiments is that in the absence of true counterfactual outcomes, they rely on parametric estimates of $\mu_t$ which we can't guarantee are well specified. An ideal dataset for evaluating true welfare impacts would have a structure that identifies $P_{X_i}$ and $\mu_t(x)$, where $\mu_t(x)$ ideally matches the agent's \textit{perceived} mean conditional potential outcomes. While we can't guarantee that our estimated imputed values for $\mu_t(x)$ match real providers, our experiments provide a structure by which such experiments could be run in the future if regulatory agencies have internal access to more ideal data, perhaps even through surveying providers themselves. 


\subsection{Calculating $\gamma_{\text{marg}}$}

When simulating information asymmetry, we compute empirical estimates of $\gamma_{\text{marg}}$ over the data. For each individual feature in Figures \ref{fig:ist_gammas_regrets_cumul} and \ref{fig:hc_gammas_regrets_cumul}, let $X$ represent the individual feature, and let $U$ represent the set of all other features. We use the full regression result as $\mu_0(x, u)$. To calculate $\mu_0(x)$ where $x$ represents the individual feature, we take the empirical conditional mean over the dataset:
\[\mu_0(x) = \frac{\sum_{i=1}^n \mu_0(X_i, U_i) \Ind(X_i = x)}{\Ind(X_i = x)}.\]

Then, $\gamma_{\text{marg}}$ is calculated empirically over the dataset as
\[\gamma_{\text{marg}} = \frac{1}{n}\sum_{i=1}^n |\mu_0(X_i) - \mu_0(X_i, U_i)|. \]

We take these empirical estimates as given in the absence of a closed form for the joint distribution of all features. More sophisticated Bayesian distribution estimation or smoothing may produce different $\gamma_{\text{marg}}$ estimates.

\subsection{Additional results}

Table \ref{tab:regrets_full_info_app} shows the utility and regret comparisons for different reward functions, including an additional comparison to Reward Function 5 in the Appendix, $w_{\text{TO}}$.

Figure \ref{fig:regrets_indiv} shows the regrets when the principal knows only each individual feature. Interestingly, use of several of these individual features leads to worse regret than if the principal knew no features and just estimated the marginal expected untreated potential outcome $E[Y(0)]$. This confirms that the regret amplification effect can occur for features other than the demographic information in Table \ref{tab:regrets_full_info}. 

\begin{table*}[!ht]
  \caption{Utility and regret comparisons for different reward functions. For each reward function $w$, we report utility $V(\pi^w)$, regret $R(\pi^w)$, and the realized treatment rate $P(T^{\pi^w} = 1)$.}
  \label{tab:regrets_full_info_app}
  \centering
  \begin{tabular}{lllllll}
    \toprule
    & \multicolumn{3}{c}{Horse Colic dataset}   &  \multicolumn{3}{c}{Stroke Trial dataset}     \\
    \cmidrule(r){2-7}
    Reward function & Utility  & Regret  & Treat rate & Utility & Regret & Treat rate \\
    \midrule
    $w_{\text{ATO}}$ & $0.00000$ & $0.1477$  & $0.1922$ & $0.00004$ & $0.0251$ & $0.0001$ \\
    $w_{\text{ATT}}$ & $0.00784$ & $0.1399$ & $0.0039$ & $0.00013$ & $0.0250$ & $0.0001$ \\
    $w_{\text{TO}}$ & $0.08568$ & $0.0621$ & $0.6706$ & $-0.08278$ & $0.1079$ & $0.6689$  \\
    $w_{\text{TT}}$ & $0.14774$ & $0.000$ & $0.2431$ & $0.02518$ & $0.000$ &  $0.1872$\\
    \midrule
    $w_{\text{TT}}$ (no info) & $0.10008$ & $0.0476$ & $0.6235$ & $-0.04888$ & $0.0741$ &  $0.4829$\\
    $w_{\text{TT}}$ (demographic info) & $0.10008$ & $0.0476$ & $0.6235$ & $-0.06391$ & $0.0891$ & $0.5041$ \\
    \bottomrule
  \end{tabular}
\end{table*}

\begin{figure}[!ht]\label{fig:regrets_indiv}
  \centering
  \begin{tabular}{ccc}
    Horse Colic dataset & Stroke Trial dataset \\
\begin{minipage}[t]{.5\textwidth}
\vspace{0pt}
\raggedleft
\includegraphics[width=\textwidth]{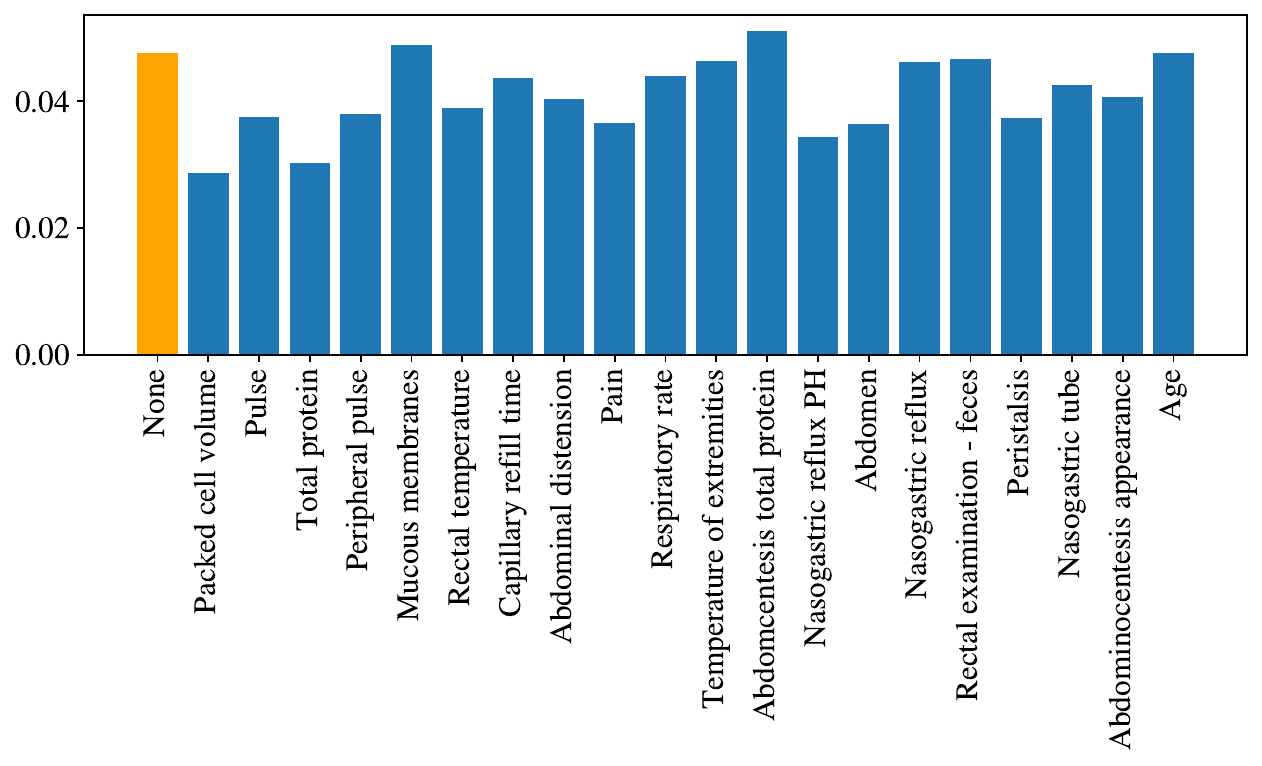}
\vspace{0pt}
\end{minipage}
       & 
\begin{minipage}[t]{.5\textwidth}
\vspace{0pt}
\raggedleft
\includegraphics[width=\textwidth]{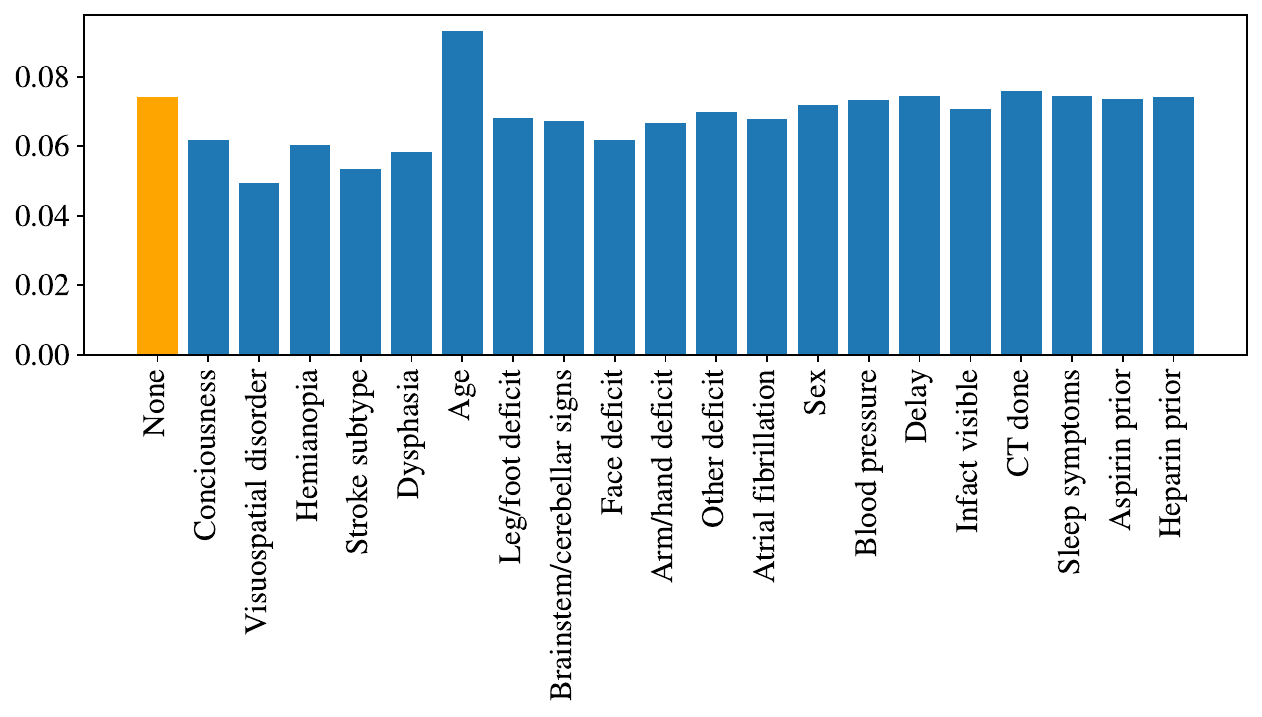}
\vspace{0pt}
\end{minipage}
\\
  \end{tabular}
  
  \caption{Regret if the principal only knows the labeled feature. Features are sorted in order of their individual $\gamma_{\text{marg}}$.}
\end{figure}

Figure \ref{fig:hc_gammas_regrets_cumul} is the equivalent of Figure \ref{fig:ist_gammas_regrets_cumul} for the Horse Colic dataset. For the Horse Colic dataset, the regret drops to close to zero if the principal knows only a few features. This is likely due to the fact that there is little variation in other feature values in the dataset conditioned \textit{packed cell volume} and \textit{pulse}, exacerbated by the fact that the dataset is so small. The value of $\gamma_{\text{marg}}$ for these two features combined is small compared to the individual $\gamma_{\text{marg}}$ values in Figure \ref{fig:hc_gammas_regrets_cumul} (\textit{left}), at $0.035$. 

This may change for larger datasets, where it may be more likely to observe more samples with the same \textit{packed cell volume} and \textit{pulse} values alongside more variation in other feature values. If in practice other features did not vary much given \textit{packed cell volume} and \textit{pulse}, then the regret shape in Figure \ref{fig:hc_gammas_regrets_cumul} (\textit{right}) would hold. More data would be needed to verify the joint distribution of $X, U$ to confirm this. 

The Stroke Trial dataset does not exhibit a similar effect since its three numerical features take significantly fewer possible values relative to the size of the dataset. 

\begin{figure}[!ht]\label{fig:hc_gammas_regrets_cumul}
  \centering
  \begin{tabular}{ccc}
    $\gamma_{\text{marg}}$ per feature & Regret for increasing feature sets \\
\begin{minipage}[t]{.45\textwidth}
\vspace{0pt}
\raggedleft
\includegraphics[trim=0 8 5 5,clip,width=\textwidth]{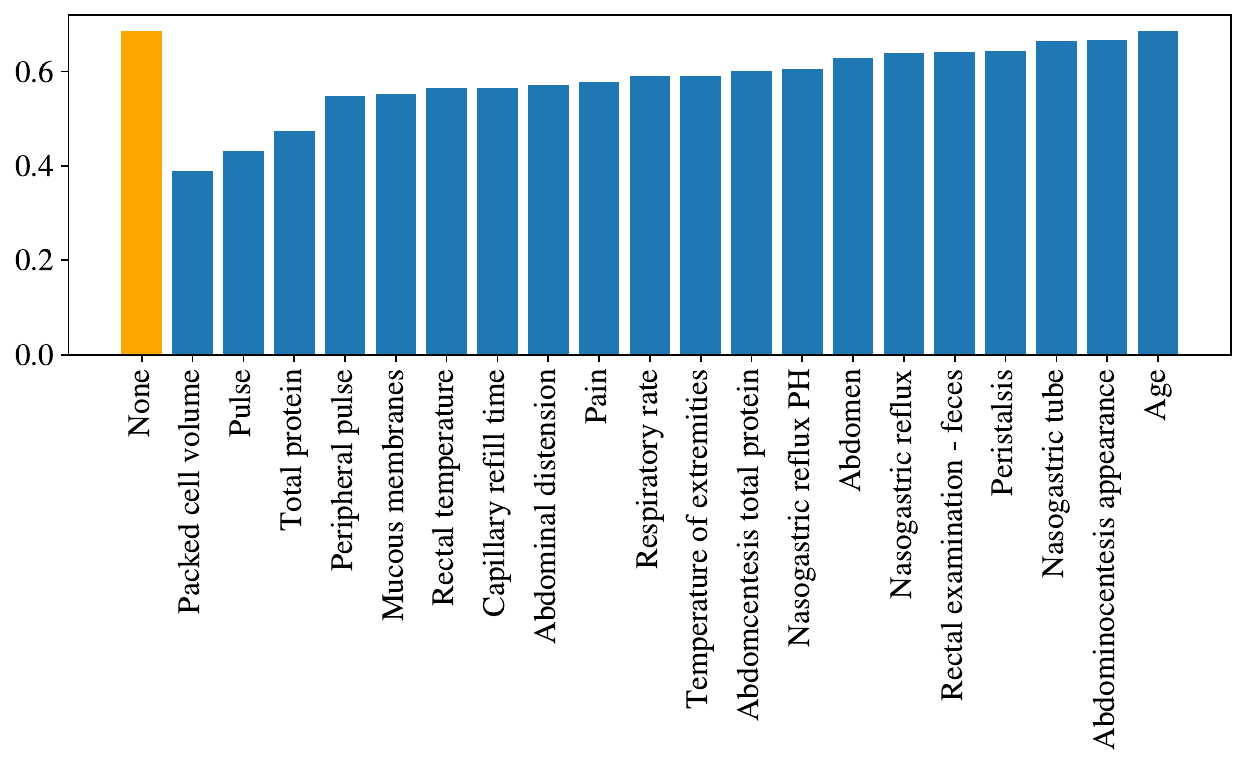}
\vspace{0pt}
\end{minipage}
       & 
\begin{minipage}[t]{.45\textwidth}
\vspace{0pt}
\raggedleft
\includegraphics[trim=0 8 5 5,clip,width=\textwidth]{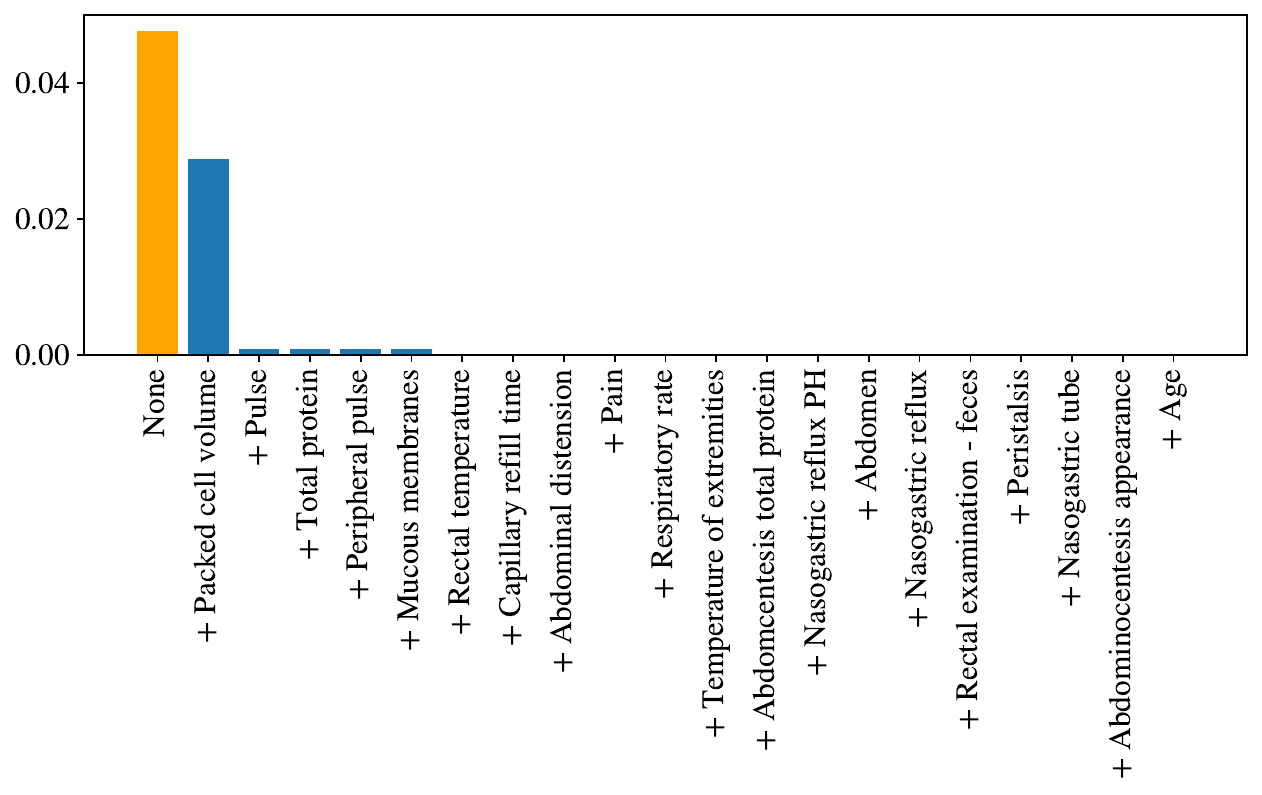}
\vspace{0pt}
\end{minipage}
\\
  \end{tabular}
  
  \caption{Regret under fine-grained information asymmetry on the Horse Colic dataset. The \textit{left} plot shows $\gamma_{\text{marg}}$ values if the principal only knows each individual feature. The \textit{right} plot shows regret as the principal accumulates features from the left (most important).}
\end{figure}

  

  

\section{Additional Discussion of Modeling Assumptions}

We do not explicitly model the agent's costs, and all budget constraints are contained in the agent's feasible set of treatment rules $\Pi$. This may be viewed as considering the agent's incentives when there is no cost to treatment, thus decoupling our analysis of the agent's incentives to do well on healthcare report cards from the broader treatment pricing market. 
In practice, the compensation hospitals and doctors receive for treatment could yield a positive adjustment or sometimes a negative adjustment to their utility per treated unit. Analysis of how other external incentives pair with the quality measure incentives would be interesting future work, but for now we focus on the incentives induced solely by the accountability metrics. 

Whether the information asymmetry modeled here is common or severe in the medical space has been debated and may change over time. \citet{dranove2003more} remark that ``providers may be able to improve their ranking by selecting patients on the basis of characteristics that are unobservable to the analysts but predictive of good outcomes.'' However, medical treatment protocols are also generally heavily codified, with online clinical decision support tools becoming increasingly widely used \cite{uptodate}.
Still, if the regulatory agency does not have full access to patents' medical records, then information asymmetry will arise.

\end{document}